%% file: main.tex
\documentclass[english]{article}

\usepackage[margin=1in]{geometry}
\usepackage[T1]{fontenc}
\usepackage{lmodern}
\usepackage{microtype}
\usepackage[utf8]{inputenc}
\input{math_commands.tex}

\usepackage{amssymb,amsthm,dsfont}
\usepackage[table]{xcolor}
\usepackage{booktabs,multirow}
\usepackage{algorithm,algorithmic}
\renewcommand{\algorithmiccomment}[1]{\hfill\(\triangleright\)~#1}
\usepackage{pgfplots}
\pgfplotsset{compat=1.18}
\usetikzlibrary{arrows.meta,positioning,fit,calc}
\usepackage{placeins,graphicx}
\usepackage{fontawesome5}
\usepackage[authoryear,round]{natbib}
\usepackage{xurl}
\usepackage[unicode,hypertexnames=false,colorlinks=true,
  linkcolor=blue,citecolor=blue,urlcolor=blue]{hyperref}
\hypersetup{
  pdftitle={One Call Ahead: Packed Self-Speculative Decoding with Diffusion Drafting},
  pdfauthor={Yuxiao Chen and Changxiao Cai},
  pdfsubject={Packed self-speculative decoding for diffusion language models}
}
\usepackage[capitalize,noabbrev]{cleveref}

\definecolor{packedgreen}{HTML}{2E8B57}
\definecolor{linearblue}{HTML}{4C78A8}
\definecolor{quadorange}{HTML}{F58518}
\definecolor{bonusviolet}{HTML}{9467BD}
\newcommand{\method}{Packed-SS}
\newcommand{\linear}{Linear-SS}
\newcommand{\quadratic}{Quadratic-SS}
\newcommand{\setdlm}{SetDLM}
\newcommand{\masktok}{\textsc{Mask}}

\usepackage{aliascnt}
\newcommand{\sharedtheorem}[3]{%
  \newaliascnt{#1}{theorem}%
  \newtheorem{#1}[#1]{#2}%
  \aliascntresetthe{#1}%
  \crefname{#1}{#2}{#3}\Crefname{#1}{#2}{#3}}
\theoremstyle{plain}
\newtheorem{theorem}{Theorem}[section]
\sharedtheorem{proposition}{Proposition}{Propositions}
\sharedtheorem{lemma}{Lemma}{Lemmas}
\sharedtheorem{corollary}{Corollary}{Corollaries}
\theoremstyle{definition}
\sharedtheorem{assumption}{Assumption}{Assumptions}
\theoremstyle{remark}
\sharedtheorem{remark}{Remark}{Remarks}
\crefname{theorem}{Theorem}{Theorems}\Crefname{theorem}{Theorem}{Theorems}
\crefname{algorithm}{Algorithm}{Algorithms}\Crefname{algorithm}{Algorithm}{Algorithms}

\title{One Call Ahead: Packed Self-Speculative\\
  Decoding with Diffusion Drafting}
\author{Yuxiao Chen\thanks{\raggedright Department of Industrial and Operations Engineering,
  University of Michigan, Ann Arbor, USA;
  \mbox{Email: \href{mailto:yuxiaoc@umich.edu}{yuxiaoc@umich.edu},
  \href{mailto:cxcai@umich.edu}{cxcai@umich.edu}.}}
  \and Changxiao Cai\footnotemark[1]}
\begin{document}
\maketitle
\input{sections/abstract}
\tableofcontents

\input{sections/introduction}

\input{sections/background}

\input{sections/related_work}
\input{sections/method}

\input{sections/experiments}
\input{sections/conclusion}

\section*{Acknowledgements}
C.\ Cai is supported by the NSF CAREER Award CCF-2541600 and the NSF grant DMS-2515333.

\bibliographystyle{apalike}
\bibliography{references}

\clearpage
\appendix
\crefalias{section}{appendix}
\crefalias{subsection}{appendix}
\crefalias{subsubsection}{appendix}
\crefname{appendix}{Appendix}{Appendices}\Crefname{appendix}{Appendix}{Appendices}
\input{sections/appendix}
\end{document}

%% file: math_commands.tex
\usepackage{amsmath,amsfonts,bm}

\def\eqref#1{equation~\ref{#1}}

\def\1{\bm{1}}

\DeclareMathAlphabet{\mathsfit}{\encodingdefault}{\sfdefault}{m}{sl}
\SetMathAlphabet{\mathsfit}{bold}{\encodingdefault}{\sfdefault}{bx}{n}

\newcommand{\E}{\mathbb{E}}

%% file: sections/abstract.tex
\begin{abstract}
Autoregressive (AR) language models generate tokens sequentially, so generating long outputs is costly.
Speculative decoding mitigates this by using a lightweight drafter to propose multiple tokens for the target AR model to verify in parallel, while a diffusion-based drafter can produce an entire draft block in one forward pass.
In models with AR and diffusion modes, one network both drafts and verifies, but drafting still waits for verification, so the standard schedule needs two sequential model calls per iteration. Existing one-call schedules avoid the wait but incur input size quadratic in the speculative width.
Motivated by this efficiency bottleneck, we introduce Packed-SS, a self-speculative decoding schedule that verifies and drafts in a single model call with input size linear in the speculative width, while preserving the target model's distribution.
The key idea is to let the diffusion drafter run \emph{one call ahead}: while verifying the current draft, the model simultaneously constructs the next draft from an older verified prefix and keeps its unresolved suffix for the next iteration.
Under a mild monotonicity assumption on the drafter, we further show that Packed-SS requires no more model calls than the standard two-call decoding in the greedy setting.
Experiments on Nemotron 3B/8B and SetDLM 1.7B show higher throughput and lower mean wall time than both baselines under greedy decoding.
\end{abstract}

\begin{center}
\footnotesize
\mbox{\faGithub\enspace\href{https://github.com/cyx78/One-Call-Ahead-Packed-Self-Speculative-Decoding-with-Diffusion-Drafting}{\nolinkurl{https://github.com/cyx78/One-Call-Ahead-Packed-Self-Speculative-Decoding-with-Diffusion-Drafting}}}
\end{center}

%% file: sections/introduction.tex
\providecommand{\nfe}{\mathrm{NFE}}
\providecommand{\indicator}{\mathbb{I}}
\providecommand{\E}{\mathbb{E}}

\section{Introduction}
\label{sec:introduction}

Modern language models (LMs) increasingly generate long sequences of intermediate reasoning before producing a final answer, making inference latency a central bottleneck. 
Standard autoregressive (AR) models \citep{brown2020gpt3} generate one token at a time from left to right. Producing an $L$-length sequence therefore requires at least $L$ sequential model calls, where the computation for each new token depends on all previously generated tokens.

Speculative decoding \citep{leviathan2023speculative,chen2023speculative,xia2023specdec} accelerates AR inference by using a lightweight \textit{draft model} that proposes multiple tokens for the full-size \textit{target model} to verify \textit{in parallel}. A rejection-sampling correction step ensures that the resulting sequence follows exactly the target model's distribution.
The target model can thus output several tokens from a single verification call. However, when the drafter is AR, creating a draft block remains sequential.
Therefore, speculative decoding shortens the expensive serial chain, but does not eliminate it.

Masked diffusion language models (MDLMs) provide a natural way to parallelize drafting \citep{sahoo2024mdlm,shi2024md4,gong2025diffullama,nie2025llada,ye2025dream}.
By learning to reverse a corruption process that transforms data into noise, an MDLM can predict multiple tokens simultaneously.
Therefore, when used as a speculative drafter, it can draft an entire block in one forward pass \citep{christopher2024specdiff,li2025diffuspec,chen2026dflash,zhang2026dflow}.
Some LMs support both causal AR prediction and parallel masked diffusion prediction in the same network \citep{liu2025tidar,fu2026nemotron}.
They provide a natural form of self-speculative decoding, with the drafter and verifier sharing the same model weights.

However, shared weights do not eliminate the sequential dependency between drafting and verification.
The standard self-speculative schedule, referred to as \textit{\linear{}}, uses two sequential calls: one AR call verifies the current draft and one diffusion call constructs the next draft from the newly verified prefix.
Drafting cannot begin until verification finishes to determine where the verified prefix ends.

\textit{\quadratic{}} \citep{liu2025tidar} eliminates this dependency by preparing a continuation for every possible verification outcome.  
With speculative width $W$, it pre-drafts separately for the cases in which $0,1,\ldots,W$ draft tokens are accepted.
Whichever outcome occurs, the next draft is already available.
This reduces each iteration to one call, but the number of evaluated positions grows quadratically with $W$.
Thus, \linear{} evaluates \(O(W)\) positions in two calls, whereas \quadratic{} evaluates \(O(W^2)\) in one.

At small batch sizes, a substantial fraction of a call's latency can come from loading model weights from memory, making it attractive to evaluate more positions per call \citep{cai2024medusa,fu2024lookahead}. 
To illustrate this tradeoff,  
suppose a call evaluating \(n\) positions costs \(\tau+\delta n\), where \(\tau\) is a fixed overhead, mostly weight loading, and \(\delta\) is the marginal cost per position.
An iteration costs approximately \(2\tau+2\delta W\) for Linear-SS and \(\tau+\delta W^2\) for Quadratic-SS. Verifying and drafting in one call with approximately \(2W\) evaluated positions would cost only \(\tau+2\delta W\) per iteration, avoiding both the second-call overhead and the quadratic branching cost.

This raises the central question of this work:
\begin{center}
   \textit{Can we verify and draft in one call, while only evaluating positions linear in the speculative width?}
\end{center}

\paragraph{Contributions.}
We answer affirmatively with \textit{\method{}}, a plug-and-play self-speculative decoding schedule for such models.
The idea is simple: the next draft does not need to wait for the newest verified prefix.
Instead, we let the diffusion drafter run \emph{one call ahead}.
At iteration $t$, one forward pass simultaneously (i) verifies the draft, and (ii) proposes a new draft block using the prefix available at the beginning of iteration $t$.
If verification advances by $A$ positions, the first $A$ positions of this new draft have already been resolved and are discarded. The remaining suffix becomes the draft for the next iteration. Drafting therefore no longer waits for verification. Figure~\ref{fig:packed-overview} illustrates the difference between \method{} and existing schedules.

The two key observations behind this design are as follows.
First, a diffusion drafter predicts tokens independently given its conditioning prefix, so a rejection at one position does not invalidate later draft tokens. 
Second, draft freshness affects speed rather than correctness: standard acceptance and residual correction remain exact when they use the proposal distribution that generated each draft token.
Thus, an older draft may be accepted less often, but does not bias the output.

These observations remove the synchronization between drafting and verification.
With an appropriate attention mask, they can be packed into one forward pass.
Each iteration thus requires only one model call, while the number of evaluated positions remains linear in \(W\). 

Our contributions are summarized as follows:
\begin{enumerate}
\item We introduce \textit{\method{}}, which simultaneously verifies the previous draft and constructs the next draft in a single model call, with input size linear in the speculative width, while preserving the target model's distribution.
\item We theoretically prove that, assuming that correct draft predictions remain correct as verified context grows, \method{} requires no more model calls than \linear{} to generate any prescribed output length.
\item We evaluate \method{} on Nemotron 3B/8B \citep{fu2026nemotron} and \setdlm{} 1.7B \citep{arriola2026setdiffusion}.
Across the primary greedy evaluations, \method{} achieves higher throughput than \linear{} and \quadratic{}.
Under temperature-1 sampling, it consistently outperforms \linear{} while remaining competitive with \quadratic{}.
\end{enumerate}

\begin{figure}[!t]
\centering
\input{figures/schedule_overview}
\caption{
\textbf{Where drafting happens in one call.} Outlined boxes are model calls.
\linear{} must finish verifying before it can draft, so an iteration costs two calls. \quadratic{} fits in one call by drafting a continuation for every place verification might stop, so its input grows quadratically with the speculative width \(W\). \method{} also uses one call, building a single draft from the previously verified prefix alone. 
The right column gives total draft and verifier prediction counts per iteration, including predictions that may later be discarded.
}
\label{fig:packed-overview}
\end{figure}
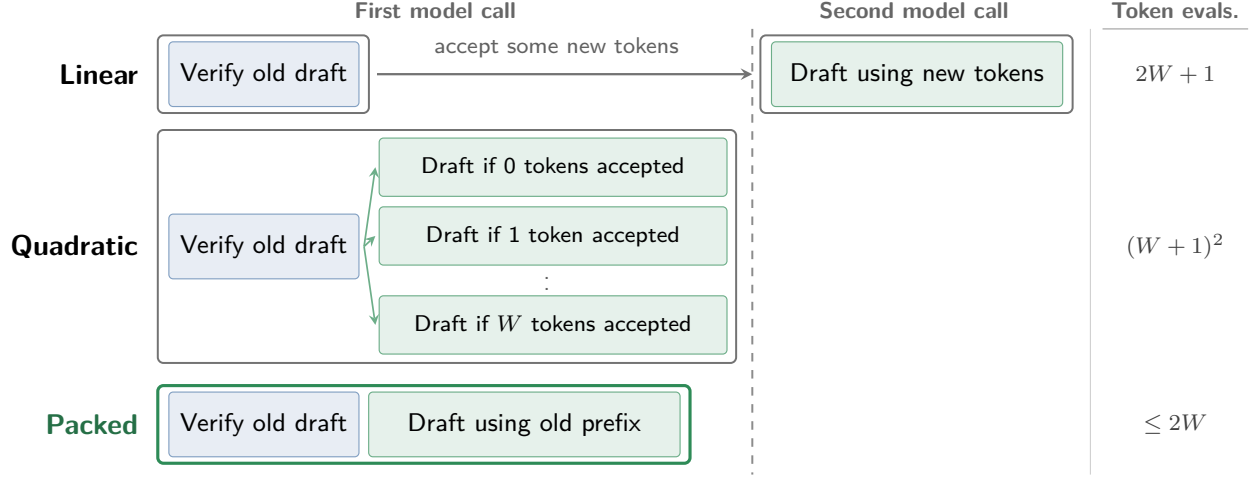

%% file: figures/schedule_overview.tex
\resizebox{\textwidth}{!}{%
\begin{tikzpicture}[x=1cm,y=1cm,font=\sffamily\small,
  call/.style={draw=black!55,rounded corners=2pt,line width=.7pt},
  ours/.style={call,draw=packedgreen,line width=1.1pt},
  reg/.style={rounded corners=1.5pt,anchor=west,align=center,inner sep=.10cm},
  ver/.style={reg,fill=linearblue!13,draw=linearblue!70,minimum height=.80cm,
              minimum width=2.35cm},
  dra/.style={reg,fill=packedgreen!12,draw=packedgreen!70,minimum height=.80cm},
  brn/.style={reg,fill=packedgreen!12,draw=packedgreen!70,minimum height=.72cm,
              minimum width=4.30cm,inner sep=.09cm,font=\sffamily\footnotesize},
  hdr/.style={anchor=center,font=\sffamily\footnotesize\bfseries,text=black!60,
              text height=1.5ex,text depth=.4ex},
  rlab/.style={anchor=east,font=\sffamily\bfseries},
  cst/.style={anchor=center,font=\sffamily\footnotesize,text=black!75},
  note/.style={font=\sffamily\footnotesize,text=black!62},
  flow/.style={->,>=stealth,line width=.8pt,draw=black!55},
  fan/.style={->,>=stealth,line width=.6pt,draw=packedgreen!70}]

\def\xL{1.57}\def\yl{5.05}\def\yq{2.92}\def\yp{0.72}
\def\xD{8.92}\def\xC{13.10}
\def\ccost{14.15}\def\yhdr{5.83}

\node[hdr] at (5.00,\yhdr) {First model call};
\node[hdr] at (10.92,\yhdr) {Second model call};

\node[rlab] at (1.45,\yl) {Linear};
\draw[call] (\xL,\yl-.48) rectangle (4.18,\yl+.48);
\node[ver] at (1.70,\yl) {Verify old draft};
\draw[flow] (4.28,\yl) -- (8.92,\yl);
\node[note,anchor=south] at (6.50,\yl+.08) {accept some new tokens};
\draw[call] (9.02,\yl-.48) rectangle (12.88,\yl+.48);
\node[dra,minimum width=3.60cm] at (9.15,\yl) {Draft using new tokens};

\node[rlab] at (1.45,\yq) {Quadratic};
\draw[call] (\xL,1.48) rectangle (8.72,4.36);
\node[ver] at (1.70,\yq) {Verify old draft};
\foreach \y in {3.90,3.05,1.95}{\draw[fan] (4.12,\yq) -- (4.26,\y);}
\node[brn] at (4.30,3.90) {Draft if 0 tokens accepted};
\node[brn] at (4.30,3.05) {Draft if 1 token accepted};
\node[note] at (6.38,2.50) {$\vdots$};
\node[brn] at (4.30,1.95) {Draft if \(W\) tokens accepted};

\node[rlab,text=packedgreen!80!black] at (1.45,\yp) {Packed};
\draw[ours] (\xL,\yp-.48) rectangle (8.15,\yp+.48);
\node[ver] at (1.70,\yp) {Verify old draft};
\node[dra,minimum width=3.85cm] at (4.17,\yp) {Draft using old prefix};

\draw[black!45,dashed,line width=.7pt] (\xD,5.61) -- (\xD,0.10);

\draw[black!18] (\xC,5.61) -- (\xC,0.10);
\node[hdr] at (\ccost,\yhdr) {Token evals.};
\draw[black!30,line width=.5pt] (13.25,5.61) -- (15.05,5.61);
\node[cst] at (\ccost,\yl) {\(2W+1\)};
\node[cst] at (\ccost,\yq) {\((W+1)^2\)};
\node[cst] at (\ccost,\yp) {\(\leq2W\)};
\end{tikzpicture}}

%% file: sections/background.tex
\section{Background}
\label{sec:background}

\subsection{Notation}
\label{sec:notation}
We write \(\mathcal V\) for the vocabulary and \(\theta\) for model parameters.
For a sequence \(x\) we write \(x^{(<i)}=(x^{(1)},\ldots,x^{(i-1)})\), and \((x,y)\) denotes the sequence \(x\) followed by item \(y\).
We use \([u]_+=\max\{u,0\}\) to denote the positive part of \(u\). A complete list of variables can be found in \Cref{app:notation}.

\subsection{Preliminaries}
\label{sec:preliminaries}

\paragraph{Speculative decoding.}
Given prefix $P$, the drafter proposes \(W\) draft tokens \(D=(d^{(1)},\ldots,d^{(W)})\), drawing each \(d^{(i)}\) from a drafter distribution \(q^{(i)}\), where \(i\) counts positions after \(P\). In one model call, the target model \(p_\theta\) uses $P$ and $D$ to return target distributions \(p^{(i)}=p_\theta(\cdot\mid P,d^{(<i)})\) for every \(i\leq W+1\).
The draft tokens are then verified from left to right, where each draft token \(d^{(i)}\) is accepted with probability \(\min\{1,\,p^{(i)}(d^{(i)})/q^{(i)}(d^{(i)})\}\).
At the first rejection, the decoder replaces that draft token with a correction token from the residual distribution \(\propto[p^{(i)}-q^{(i)}]_+\). If all tokens in a draft are accepted, the target model yields one extra bonus token drawn from \(p^{(W+1)}=p_\theta(\cdot\mid P,D)\).
The rejection-sampling correction ensures that the generated text follows the target distribution exactly, so a low-quality drafter costs speed but never quality.

In practice, we use the target model to sample at a certain temperature, with temperature \(1\) being the model's original distribution and lower temperature yielding a sharper distribution. At temperature \(0\), the output distribution mass is on its \(\arg\max\) token. This is also called \textit{greedy decoding}. 

\paragraph{Masked diffusion language models.}
A masked diffusion language model (MDLM) generates text by repeatedly
filling missing token positions, represented by a special mask token
\(\masktok\) \citep{sahoo2024mdlm,shi2024md4}.
Following the denoising framework of discrete diffusion
\citep{hoogeboom2021multinomial,austin2021d3pm,campbell2022continuous,lou2024sedd},
the model is trained to predict original tokens from sequences in which
randomly selected tokens have been replaced by \(\masktok\).
During inference, let \(C\) denote the visible tokens and their
positions, and \(U\) the remaining masked positions. One model call
predicts, for each position \(i\in U\), a distribution
\(q_\theta^{(i)}(\cdot\mid C)\) that approximates the conditional
marginal \(p^{(i)}(\cdot\mid C)\) of the data distribution \(p\). Sampling
each position independently from its prediction approximates the
joint distribution \(p(z_U\mid C)\) by the product of marginals
\(q_\theta(z_U\mid C)=\prod_{i\in U}q_\theta^{(i)}(z_i\mid C)\).
This approximation incurs error whenever the tokens in \(U\) are not
conditionally independent given \(C\), incurring intrinsic trade-off between sampling efficiency and accuracy \citep{li2025breaking,chen2025optimal,zhao2026adaptation}. 
To keep this error small, the decoder fills only a subset of the positions in each call, often those with the most confident predictions
\citep{ben2025accelerated,nie2025llada,cai2026confidence}, and then
repeats the model call with the updated context to predict the
remaining positions.
Block diffusion applies this procedure within blocks generated from left
to right, completing each block before starting the next
\citep{arriola2025blockdiffusion}. A block size of one recovers
autoregressive generation, while a single block covering the full
sequence recovers full-sequence masked diffusion.

\paragraph{Shared-weight drafting and verification.}
A model with shared-weight AR and diffusion modes can perform both roles
in speculative decoding: its masked mode serves as the \emph{drafter}
\(q_\theta\), proposing multiple tokens after the verified prefix in
parallel, while its causal mode serves as the \emph{verifier} \(p_\theta\).
When these modes differ only in their attention masks, \method{} can
combine them in one forward pass using a custom attention mask
\citep{miao2024specinfer,spector2023staged,fu2024lookahead}.
Nemotron-Labs-Diffusion supports both modes through joint training on
autoregressive and diffusion objectives \citep{fu2026nemotron}.
\setdlm{} also supports this construction through its any-order
formulation: single-position conditionals in left-to-right order
provide AR verification, while parallel predictions of unresolved
positions provide diffusion drafting \citep{arriola2026setdiffusion}.
See \cref{app:implementation} for model-specific input layouts and
attention-mask constructions.

%% file: sections/related_work.tex
\subsection{Related work}
\label{sec:related_work}

\paragraph{Sequential draft--verify decoding.}
Most speculative decoding methods improve drafting within a sequential
draft--verify schedule
\citep{leviathan2023speculative,chen2023speculative}.
Self-speculative methods reuse the target network: Draft \& Verify \citep{zhang2024draftverify} skips
selected layers without retraining, whereas LayerSkip \citep{elhoushi2024layerskip} trains intermediate
layers for early-exit drafting.
EAGLE \citep{li2024eagle,li2025eagle3} variants instead train auxiliary drafters using target features.
Parallel drafting reduces sequential work, from blockwise parallel decoding \citep{stern2018blockwise}
and SpecDec \citep{xia2023specdec} to diffusion-based drafters such as SpecDiff \citep{christopher2024specdiff} and DiffuSpec \citep{li2025diffuspec}.
DEER \citep{cheng2025deer} aligns prefix-conditioned drafts with an AR teacher, while DFlash \citep{chen2026dflash} uses
a shallow block-diffusion drafter conditioned on target features.
DSpark \citep{cheng2026dspark} adds confidence-based verification-length scheduling,
while DFlow \citep{zhang2026dflow} reuses verifier features from rejected
suffixes to improve subsequent drafts.

\paragraph{Other speculative decoding schedules.}
SpecInfer \citep{miao2024specinfer} and Medusa \citep{cai2024medusa} use tree attention to verify multiple candidate
continuations, while EAGLE-2 \citep{li2024eagle2} adaptively constructs draft trees using
confidence estimates.
Tree-structured verification covers more candidate continuations, but does
not by itself remove the dependency between drafting and verification.
Jacobi decoding \citep{santilli2023parallel} takes a different approach, updating multiple tentative
positions in parallel through fixed-point iteration.
Lookahead Decoding \citep{fu2024lookahead} turns these trajectories into cached $n$-gram candidates
and combines a lookahead branch with a verification branch in one forward
pass.
For diffusion decoding, SSD \citep{gao2025ssd} and FreeDave \citep{wu2025freedave} batch multiple speculative
decoding states, with FreeDave also reusing accepted-state predictions for
subsequent drafting.
Within shared-weight models, the quadratic schedules in
\citet{samragh2025future}, TiDAR \citep{liu2025tidar}, and Nemotron \citep{fu2026nemotron} prepare continuations at every
possible acceptance boundary, combining drafting and verification at the
cost of $O(W^2)$ evaluated positions for speculative width $W$.
Apple's linear decoding \citep{samragh2025future} and Introspective DLMs \citep{yu2026introspective} discard continuations
conditioned on tentative tokens after rejection, which frequently fails to simultaneously yield a draft.
Speculative Jacobi Decoding \citep{teng2024sjd} reuses predictions after rejection, but its
proposals also condition on tentative tokens.

%% file: sections/method.tex
\section{Our Method}
\label{sec:method}

In this section, we first present the high-level description of \method{}. We then describe the procedure under greedy decoding in \cref{sec:method-greedy}, where the main idea is clearest. \cref{sec:method-general} describes the general procedure.

At a high level, iteration $t$ of \method{} reads the verified prefix \(P_{t-1}\) and the current draft \(D_{t-1}\), and performs verification and drafting simultaneously and independently in one model call.
The verifier accepts $A_t$ tokens from the standard speculative decoding process (including one correction or bonus token). These $A_t$ tokens are extended to \(P_{t-1}\) to form the next prefix \(P_t\).
For speculative width \(W\), the drafter reads only \(P_{t-1}\) and generates a new proposal \(Z_t=(z_t^{(i)})_{i=1}^W\) for the next \(W\) positions.
The first \(A_t\) of those positions have just been filled by the verifier, so their
proposal tokens are discarded, and the rest become the draft \(D_t\) that iteration \(t+1\)
verifies.
See \cref{fig:method-notation} for an illustration of one iteration under greedy decoding.

\subsection{Packed self-speculative procedure under greedy decoding}
\label{sec:method-greedy}

We now state the procedure under greedy decoding, and summarize it in \cref{alg:packed}.

\input{figures/method_notation}

\paragraph{Initialization.} Given a prompt \(P\), start with the prefix $P_0=P$ and draft \(D_0=\varnothing\).

\paragraph{Iteration \(t\geq1\).}
Given the prefix \(P_{t-1}\) and draft \(D_{t-1}=(d_{t-1}^{(1)},d_{t-1}^{(2)},\ldots)\) following $P_{t-1}$, iteration \(t\) evaluates the verifier and drafter in one forward pass.

For verification, the AR mode predicts the token at each position of the previous draft plus one position after it, conditioned on the prefix $P_{t-1}$ and the draft tokens before that position:
\begin{equation}
v_t^{(i)}=\arg\max_{x\in\mathcal V}p_\theta(x\mid P_{t-1},d_{t-1}^{(<i)}),
\qquad 1\leq i\leq|D_{t-1}|+1.
\label{eq:packed-verifier}
\end{equation}
Let \(h_t\) be the number of leading draft tokens that match the verifier's predictions:
\begin{equation}
h_t=\max\{0 \leq h \leq |D_{t-1}|:d_{t-1}^{(i)}=v_t^{(i)}\text{ for every }i\leq h\}.
\label{eq:accept}
\end{equation}
Verification yields the accepted draft prefix followed by \(v_t^{(h_t+1)}\), which corrects the first mismatch or supplies a bonus after full acceptance. Thus, a total of \(A_t=h_t+1\) tokens are accepted.

In the same call, the diffusion mode drafts the \(W\) positions that follow the previous prefix \(P_{t-1}\):
\begin{equation}
Z_t=(z_t^{(1)},\ldots,z_t^{(W)}),\qquad
z_t^{(i)}=\arg\max_{x\in\mathcal V}q_\theta^{(i)}(x\mid P_{t-1}).
\label{eq:draft}
\end{equation}
It is important to note that the draft tokens \(Z_t\) are generated conditioned on the previous prefix \(P_{t-1}\), rather than in \linear{} the new prefix \(P_t\). This distinction allows drafting and verification to run in parallel while \linear{} needs to be evaluated sequentially.
On nonterminal iterations, the suffix of \(Z_t\) after its first \(A_t\) tokens becomes the new draft:
\begin{equation}
\label{eq:draft-size-t}
D_t=(z_t^{(A_t+1)},\ldots,z_t^{(W)}).
\end{equation}

\paragraph{Termination.}
Stop at the first accepted end-of-sequence (EOS) token or after \(L\) new tokens.

\begin{algorithm}[t]
\caption{\method{} with LTR greedy decoding}
\label{alg:packed}
\small
\begin{algorithmic}[1]
\REQUIRE prompt \(P\), speculative width \(W\), output length limit \(L\), verifier \(p_\theta\), drafter \(q_\theta\)
\STATE \(P_0\leftarrow P\), \ \(D_0\leftarrow\varnothing\), \ \(t\leftarrow1\)
\WHILE{\(|P_{t-1}|-|P|<L\) and (\(t=1\) or the last token of \(P_{t-1}\) is not \(\mathrm{EOS}\))}
\item[] \textbf{In one model call:}
\begin{ALC@g}
\STATE \(v_t^{(i)}\leftarrow\arg\max_x p_\theta(x\mid P_{t-1},d_{t-1}^{(<i)})\) for \(i=1,\ldots,|D_{t-1}|+1\) \COMMENT{Verifier predictions}
\STATE \(z_t^{(i)}\leftarrow\arg\max_x q_\theta^{(i)}(x\mid P_{t-1})\) for \(i=1,\ldots,W\) \COMMENT{Drafter proposals}
\end{ALC@g}
\item[] \textbf{end}
\STATE \(h_t\leftarrow\max\{0\leq h\leq|D_{t-1}|:d_{t-1}^{(i)}=v_t^{(i)}\text{ for all }i\leq h\}\) \COMMENT{Count matches}
\STATE \(C_t\leftarrow(d_{t-1}^{(1)},\ldots,d_{t-1}^{(h_t)},v_t^{(h_t+1)})\), \ \(A_t\leftarrow h_t+1\) \COMMENT{Add verifier token}
\STATE \(a\leftarrow\min\bigl(\{A_t,L+|P|-|P_{t-1}|\}\cup\{1\leq i\leq A_t:C_t^{(i)}=\mathrm{EOS}\}\bigr)\) \COMMENT{Ending condition}
\STATE \(C_t\leftarrow(C_t^{(1)},\ldots,C_t^{(a)})\) \COMMENT{Prune by stopping index}
\STATE \(P_t\leftarrow(P_{t-1},C_t)\) \COMMENT{Prepare next prefix, append accepted tokens}
\STATE \(D_t\leftarrow(z_t^{(a+1)},\ldots,z_t^{(W)})\) \COMMENT{Prepare next draft, $z_t^{(i)}$ is generated in line 4}
\STATE \(t\leftarrow t+1\)
\ENDWHILE
\STATE \textbf{Output:} \(P_{t-1}\)
\end{algorithmic}
\end{algorithm}

\paragraph{Comparison with \linear{} and \quadratic{}.}
\linear{} alternates between drafting-only and verification-only calls,
whereas \method{} performs both in each call.
By retaining the unresolved suffix of \(Z_t\) as the next draft,
\method{} avoids a separate drafting call.
This overlap comes at the cost of a shorter draft conditioned on less
verified context: on nonterminal iterations of \cref{alg:packed},
\method{} retains \(|D_t|=W-A_t\leq W-1\) tokens by
\eqref{eq:draft-size-t}, whereas \linear{} drafts \(W\) tokens
conditioned on \(A_t\) more verified tokens.
In addition, both \method{} and \quadratic{} draft and verify simultaneously in one model call.
However, \method{} evaluates $O(W)$ positions per iteration rather than $O(W^2)$.
The trade-off is conditioning context: \quadratic{}'s selected proposal excludes only the final verifier token from $P_t$, whereas \method{} conditions on $P_{t-1}$, excluding all $A_t$ newly appended tokens.
Conceptually, \method{} can be viewed as \quadratic{} with only the branch for zero
accepted draft tokens.

\input{sections/theory}

\subsection{General packed self-speculative decoding}
\label{sec:method-general}

We now present \method{} for general sampling. Compared with the greedy procedure in \cref{alg:packed} from \cref{sec:method-greedy}, the general sampling requires two minor modifications.

Consider iteration $t$. First, we store beyond each draft token $d_t^{(i)}$ the distribution from which it was sampled, \(Q_t^{(i)}=q_t^{(a+i)}=q_\theta^{(a+i)}(\cdot\mid P_{t-1})\).
At temperature \(0\), \(Q_{t-1}^{(i)}\) collapses to a point mass on \(d_{t-1}^{(i)}\). So storing the draft token alone is sufficient as in \cref{alg:packed}.

Second, when verifying each draft token \(d_{t-1}^{(i)}\), we replace the deterministic match test \(d_{t-1}^{(i)} = v_{t}^{(i)}\) in \cref{eq:accept} with the standard rejection-sampling acceptance rule.
With \(p_t^{(i)}=p_\theta(\cdot\mid P_{t-1},d_{t-1}^{(<i)})\) the target conditional at that position and \(Q_{t-1}^{(i)}\) its stored drafter distribution, we accept the draft token \(d_{t-1}^{(i)}\) with probability \(\min\{1,p_t^{(i)}(d_{t-1}^{(i)})/Q_{t-1}^{(i)}(d_{t-1}^{(i)})\}\).
Crucially, the denominator is the stored drafter distribution \(Q_{t-1}^{(i)}\) from which $d_{t-1}^{(i)}$ was originally sampled, even though it may have been computed from an older prefix.
At the first rejection, we draw a correction from the normalized residual distribution \( \propto [p_t^{(i)}-Q_{t-1}^{(i)}]_{+}\). If the entire draft is accepted, we draw a bonus token from \(p_t^{(|D_{t-1}|+1)}\).

We summarize the resulting procedure in \Cref{alg:packed-rejection} in \cref{app:sampling}.
The following theorem, proved in \cref{app:proofs}, states that \method{} preserves the distribution of the AR mode. Since greedy decoding is recovered at temperature \(0\), it also implies the correctness of \cref{alg:packed}.
\begin{theorem}[Distributional correctness]
\label{thm:sampling}
The continuation generated by \method{} after the prompt \(P\) has the same distribution as sampling from the AR model \(p_\theta\), stopped at the first EOS or at the output length limit \(L\).
\end{theorem}

Finally, storing the drafter distributions $Q_{t-1}^{(i)}$ need not require a probability for every vocabulary token at each draft position.
With top-\(k\) \citep{fan2018topk} or top-\(p\) truncation
\citep{holtzman2020nucleus}, we need to store only the nonzero entries
of \(Q_{t-1}^{(i)}\).
For instance, Qwen3 recommends top-\(k=20\) for a vocabulary of \(151{,}936\) tokens \citep{qwen3}, and the reference implementation of Llama 3 uses top-\(p=0.9\) by default \citep{llama3}.

%% file: figures/method_notation.tex
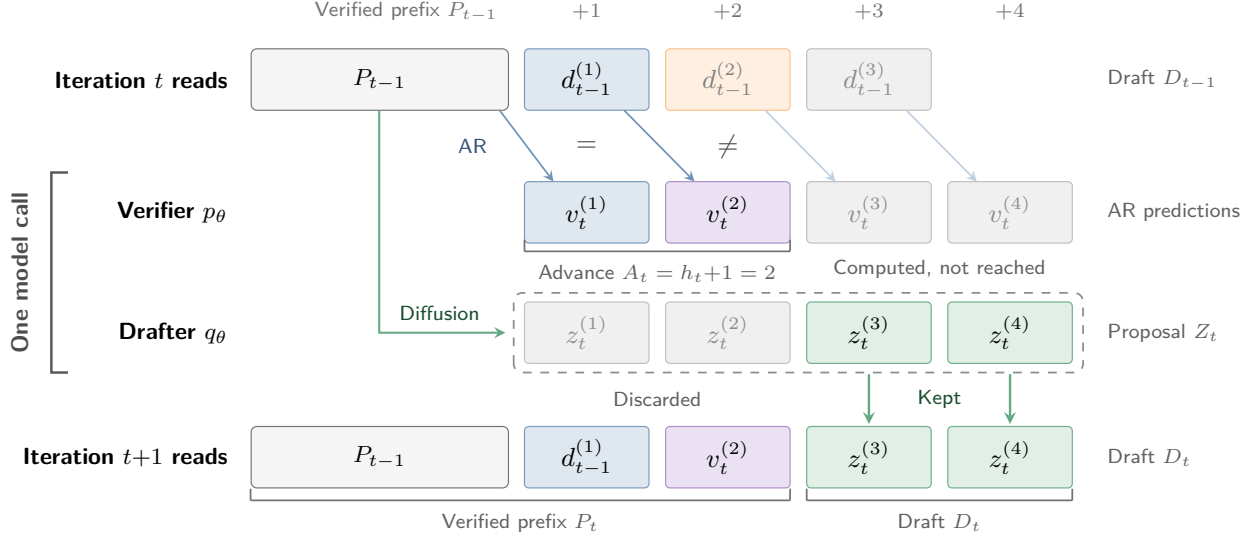
\begin{figure}[t]
\centering
\resizebox{\textwidth}{!}{%
\begin{tikzpicture}[x=1cm,y=1cm,font=\sffamily\small,
  slot/.style={draw=black!45,rounded corners=1.5pt,minimum width=1.5cm,
               minimum height=.74cm,inner sep=1pt},
  tokA/.style={slot,fill=linearblue!18,draw=linearblue!75},
  tokB/.style={slot,fill=quadorange!13,draw=quadorange!50,text=black!50},
  tokC/.style={slot,fill=bonusviolet!20,draw=bonusviolet!75},
  zslot/.style={slot,fill=packedgreen!14,draw=packedgreen!70},
  gone/.style={slot,fill=black!6,draw=black!25,text=black!45},
  pre/.style={draw=black!55,rounded corners=2pt,fill=black!4,anchor=west,
              minimum height=.74cm,minimum width=3.10cm,inner sep=2pt,
              font=\sffamily\footnotesize},
  rlab/.style={anchor=east,font=\sffamily\footnotesize\bfseries},
  hdr/.style={anchor=center,font=\sffamily\scriptsize,text=black!55},
  note/.style={font=\sffamily\scriptsize,text=black!62},
  side/.style={note,anchor=west},
  bar/.style={draw=black!55,line width=.6pt},
  ar/.style={->,>=stealth,line width=.6pt,draw=linearblue!80},
  feed/.style={->,>=stealth,line width=.7pt,draw=packedgreen!75},
  keep/.style={->,>=stealth,line width=.8pt,draw=packedgreen!75}]

\def\cA{5.60}\def\cB{7.30}\def\cC{9.00}\def\cD{10.70}
\def\ya{4.75}\def\yb{3.15}\def\yc{1.70}\def\yd{0.20}
\def\xs{11.75}\def\xp{1.55}

\draw[black!65,line width=.85pt]
  (-.65,3.63) -- (-.85,3.63) -- (-.85,1.22) -- (-.65,1.22);
\node[rotate=90,font=\sffamily\footnotesize\bfseries,text=black!70]
  at (-1.23,2.425) {One model call};

\node[hdr,anchor=east] at (4.65,5.57) {Verified prefix \(P_{t-1}\)};
\node[hdr] at (\cA,5.57) {\(+1\)};
\node[hdr] at (\cB,5.57) {\(+2\)};
\node[hdr] at (\cC,5.57) {\(+3\)};
\node[hdr] at (\cD,5.57) {\(+4\)};

\node[rlab] at (1.40,\ya) {Iteration \(t\) reads};
\node[pre] at (\xp,\ya) {\(P_{t-1}\)};
\node[tokA] at (\cA,\ya) {\(d_{t-1}^{(1)}\)};
\node[tokB] at (\cB,\ya) {\(d_{t-1}^{(2)}\)};
\node[gone] at (\cC,\ya) {\(d_{t-1}^{(3)}\)};
\node[side] at (\xs,\ya) {Draft \(D_{t-1}\)};

\node[note,anchor=east,text=linearblue!70!black] at (4.52,3.95) {AR};
\draw[ar] (4.55,\ya-.37) -- (5.20,\yb+.37);
\draw[ar] (6.05,\ya-.37) -- (6.90,\yb+.37);
\draw[ar,draw=linearblue!35] (7.75,\ya-.37) -- (8.60,\yb+.37);
\draw[ar,draw=linearblue!35] (9.45,\ya-.37) -- (10.30,\yb+.37);

\node[font=\sffamily\small,text=black!70] at (\cA,3.95) {\(=\)};
\node[font=\sffamily\small,text=black!70] at (\cB,3.95) {\(\neq\)};

\node[rlab] at (1.40,\yb) {Verifier \(p_\theta\)};
\node[tokA] at (\cA,\yb) {\(v_t^{(1)}\)};
\node[tokC] at (\cB,\yb) {\(v_t^{(2)}\)};
\node[gone] at (\cC,\yb) {\(v_t^{(3)}\)};
\node[gone] at (\cD,\yb) {\(v_t^{(4)}\)};
\node[side] at (\xs,\yb) {AR predictions};
\node[note,anchor=north] at (9.85,\yb-.45) {Computed, not reached};

\draw[bar] (4.85,2.68) -- (8.05,2.68);
\draw[bar] (4.85,2.68) -- (4.85,2.82);
\draw[bar] (8.05,2.68) -- (8.05,2.82);
\node[note,anchor=north] at (6.45,2.64) {Advance \(A_t=h_t{+}1=2\)};

\node[rlab] at (1.40,\yc) {Drafter \(q_\theta\)};
\node[gone] at (\cA,\yc) {\(z_t^{(1)}\)};
\node[gone] at (\cB,\yc) {\(z_t^{(2)}\)};
\node[zslot] at (\cC,\yc) {\(z_t^{(3)}\)};
\node[zslot] at (\cD,\yc) {\(z_t^{(4)}\)};
\draw[black!45,dashed,line width=.6pt,rounded corners=3pt]
  (4.72,\yc-.48) rectangle (11.58,\yc+.48);
\node[side] at (\xs,\yc) {Proposal \(Z_t\)};

\draw[feed] (3.10,\ya-.37) -- (3.10,\yc) -- (4.62,\yc);
\node[note,anchor=west,text=packedgreen!55!black] at (3.22,\yc+.24) {Diffusion};

\node[rlab] at (1.40,\yd) {Iteration \(t{+}1\) reads};
\node[pre] at (\xp,\yd) {\(P_{t-1}\)};
\node[tokA] at (\cA,\yd) {\(d_{t-1}^{(1)}\)};
\node[tokC] at (\cB,\yd) {\(v_t^{(2)}\)};
\node[zslot] at (\cC,\yd) {\(z_t^{(3)}\)};
\node[zslot] at (\cD,\yd) {\(z_t^{(4)}\)};
\node[side] at (\xs,\yd) {Draft \(D_t\)};

\draw[bar] (\xp,-0.32) -- (8.05,-0.32);
\draw[bar] (\xp,-0.32) -- (\xp,-0.18);
\draw[bar] (8.05,-0.32) -- (8.05,-0.18);
\node[note,anchor=north] at (4.80,-0.36) {Verified prefix \(P_t\)};
\draw[bar] (8.25,-0.32) -- (11.45,-0.32);
\draw[bar] (8.25,-0.32) -- (8.25,-0.18);
\draw[bar] (11.45,-0.32) -- (11.45,-0.18);
\node[note,anchor=north] at (9.85,-0.36) {Draft \(D_t\)};

\draw[keep] (\cC,1.20) -- (\cC,0.60);
\draw[keep] (\cD,1.20) -- (\cD,0.60);
\node[note,anchor=center] at (6.45,0.90) {Discarded};
\node[note,anchor=center,text=packedgreen!55!black] at (9.85,0.90) {Kept};
\end{tikzpicture}}
\caption{\textbf{One iteration of greedy \method{}.}
Columns are token positions. The old draft token \(d_{t-1}^{(i)}\), the verifier prediction
\(v_t^{(i)}\), and the new proposal token \(z_t^{(i)}\) are in the same position.
A single forward pass takes the verified prefix \(P_{t-1}\) and draft
\(D_{t-1}\) as input and produces the two middle rows: the AR verifier's
next-token predictions \(v_t^{(i)}\), conditioned on \(P_{t-1}\) and
\(d_{t-1}^{(1)},\dots,d_{t-1}^{(i-1)}\), and the diffusion drafter's
proposal \(Z_t=(z_t^{(i)})_{i=1}^W\), conditioned on \(P_{t-1}\) alone.
Here, one matching draft token (\(h_t=1\)) and the correction token
\(v_t^{(2)}\) yield \(A_t=2\) output tokens.
Proposals at these two positions are discarded, leaving
\(D_t=(z_t^{(3)},z_t^{(4)})\) as the next draft.}
\label{fig:method-notation}
\end{figure}

%% file: sections/theory.tex
\paragraph{Theoretical guarantees.}

Under greedy decoding and the monotonicity assumption below, \method{}
requires no more forward passes than \linear{} to generate a prescribed
number of tokens from the same prompt.

Fix a prompt and speculative width \(W\), and let \(x_1,x_2,\ldots\)
denote the target model's greedy continuation.
Let \(\widehat x_k(c)\) denote the deterministic drafter's prediction
at position \(k\), conditioned on the prompt and \(x_1,\ldots,x_c\).

\begin{assumption}[Monotonicity of draft correctness]
\label{ass:threshold}
For each output position \(k\), there exists a threshold
\(c_k^\star\leq k\) such that
$\mathds 1\{\widehat x_k(c)=x_k\}=\mathds 1\{c\geq c_k^\star\}$
for every \(c<k\) evaluated by either schedule.
\end{assumption}

Here, \(c_k^\star=k\) indicates that position \(k\) is never correctly
predicted before it is verified. Thus, additional verified context cannot turn a correct prediction
into an incorrect one.
The thresholds need not increase with \(k\): a later position may
become correctly predictable with less context than an earlier one.

\begin{theorem}[Pathwise dominance]
\label{thm:pathwise}
Suppose \method{} and \linear{} use the same target model, drafter,
and speculative width \(W\geq1\), and start from a common verified prefix.
Under \cref{ass:threshold}, \method{}'s verified prefix after \(2t\)
forward passes is at least as long as \linear{}'s after \(t\) iterations,
for every \(t\geq0\).

Let \(\nfe_S(L)\) denote the number of forward passes schedule \(S\)
requires to generate \(L\) tokens after \(P\). Since each \linear{} iteration uses two forward passes, for every \(L\geq0\), one has
\begin{equation}
\nfe_{\text{\method{}}}(L)\leq\nfe_{\text{\linear{}}}(L).
\label{eq:finite-dom}
\end{equation}
\end{theorem}

Since both schedules follow the same greedy continuation, the bound
also holds when stopping at the first emitted EOS token.
See \cref{app:pathwise} for the proof.

%% file: sections/experiments.tex
\section{Experiments}
\label{sec:experiments}

\paragraph{Models and datasets.}
We evaluate \method{} on Nemotron-Labs-Diffusion
\citep{fu2026nemotron} and \setdlm{}
\citep{arriola2026setdiffusion}.
Nemotron 3B and 8B are evaluated on five datasets:
GSM8K \citep{cobbe2021gsm8k} and MATH-500
\citep{hendrycks2021math} for mathematical reasoning,
HumanEval \citep{chen2021humaneval} and MBPP
\citep{austin2021mbpp} for code generation, and
WritingBench \citep{wu2025writingbench} for generation speed
on writing prompts.
For \setdlm{}, we evaluate two separately fine-tuned checkpoints
on GSM8K released by \citep{arriola2026setdiffusion}.

\paragraph{Experimental setup.}
All experiments use batch size one on a single A100 80GB GPU,
with Nemotron in BF16 and \setdlm{} in FP32.
Generations are capped at 512 tokens for code and writing
and 1,024 tokens for mathematics, to avoid repetitive responses dominating the comparison.
Because output lengths can differ across schedules, our primary
efficiency metric is aggregate token throughput, which decomposes as
\[
\text{Throughput}
= \frac{\text{total generated tokens}}{\text{total decoder wall time}}
= \frac{\text{Tokens/NFE}}{\text{Time/NFE}}.
\]
Tokens per NFE (Tokens/NFE) and decoder time per NFE (Time/NFE) are recorded to separate token yield from cost per NFE.
Detailed setup and results are provided in
\cref{app:experimental-details,app:records}.
\subsection{Greedy decoding}
\label{sec:exp-greedy}
\paragraph{Nemotron.}
\linear{} and \method{} use diffusion block size \(B=32\),
while \quadratic{} uses \(B=8\).
Each queried block includes the final token of the verified prefix,
so the effective speculative width is \(W=B-1\).
This gives \method{} and \quadratic{} comparable maximum
input sizes of 64 and 72 token rows, respectively. These block sizes also give the highest throughput among the settings tested in the small-scale experiment
in \cref{app:width}.

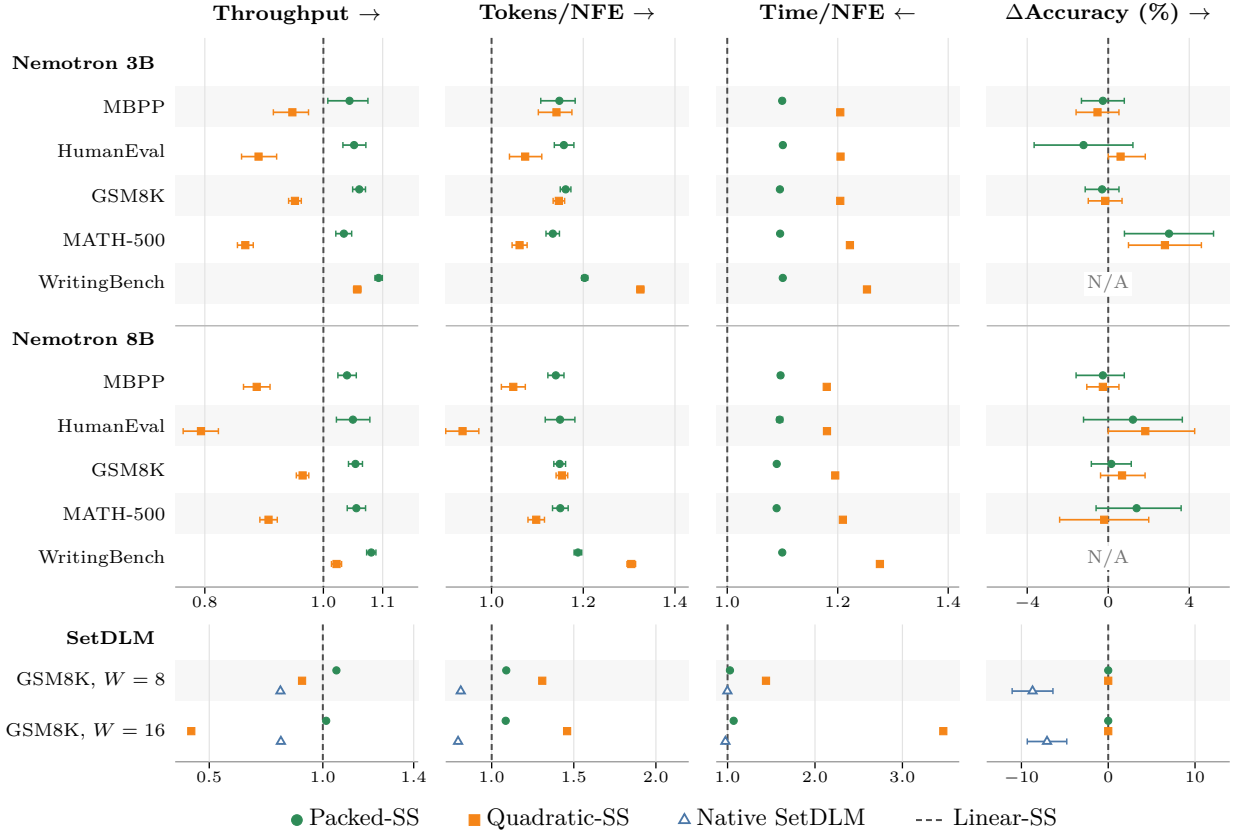
\begin{figure}[!t]
\centering
\input{figures/efficiency_greedy}
\caption{\textbf{Greedy decoding efficiency and accuracy.}
Efficiency metrics are ratios to \linear{}; accuracy differences
are in percentage points.
Error bars show 95\% paired bootstrap confidence intervals.
Native is the released non-speculative \setdlm{} diffusion decoder,
shown for reference. WritingBench measures speed only.}
\label{fig:efficiency-greedy}
\end{figure}

Results are summarized in \cref{fig:efficiency-greedy}, with numerical results in \cref{tab:nemotron}.
\method{} achieves the highest measured throughput in all ten
Nemotron settings.
Compared with \linear{}, throughput improves by \(3.5\)--\(9.3\%\),
with all ten 95\% confidence intervals for the throughput
ratio above one.
The improvement comes from a \(13.4\)--\(20.3\%\) increase in
tokens per NFE, which more than offsets the
\(8.9\)--\(10.1\%\) increase in decoder time per NFE.
Compared with \quadratic{}, throughput improves by
\(3.4\)--\(32.3\%\) at the evaluated block sizes.
Accuracy differences between schedules are observed, but they do not
provide a direct comparison of decoding quality because they arise from
finite-precision effects.

\paragraph{\setdlm{}.}
The checkpoints with training horizons \(S_{\max}=16\) and
\(S_{\max}=32\) are evaluated at speculative widths \(W=8\)
and \(W=16\), respectively.
All three speculative schedules use the same width within each
checkpoint.

\method{} improves throughput over \linear{} by \(6.0\%\) at
\(W=8\) and \(1.5\%\) at \(W=16\), as in \cref{fig:efficiency-greedy,tab:setdlm-length-details}.
The additional token yield outweighs the
higher cost per NFE: tokens per NFE increase by
\(8.5\)--\(8.9\%\), while decoder time per NFE increases
by \(2.7\)--\(7.0\%\).
\method{} also has higher throughput than \quadratic{}
in both settings.
The improvements extend to confidence-based verification order;
details are provided in \cref{app:setdlm-conf}.

We also test the native \setdlm{} decoding schedule, which is diffusion decoding instead of
speculative decoding.
At \(W=8\) and \(W=16\), \method{} achieves higher throughput
than Native, with accuracy higher by \(8.72\) and \(7.05\)
percentage points, respectively.

\subsection{Decoding at temperature 1}
\label{sec:exp-sampling}

We further evaluate the sampling procedure of
\cref{sec:method-general} at temperature 1, where draft tokens are
verified probabilistically.
We use the same five datasets and both Nemotron scales as in greedy
decoding, with block size \(B=32\) for \linear{} and \method{} and
\(B=8\) for \quadratic{}.

\paragraph{Scope and baselines.}
\linear{} with speculative sampling is the standard stochastic setting
\citep{leviathan2023speculative,chen2023speculative}.
However, prior empirical evaluations of \quadratic{} are reported
under greedy decoding
\citep{liu2025tidar,samragh2025future,fu2026nemotron}.
We therefore use \linear{} as the primary baseline at temperature 1.
For completeness, we extend \quadratic{} to rejection sampling by
applying the same probabilistic verification to its selected branch
as a secondary comparison. More details are in \cref{app:sampling-protocol}.

\begin{figure}[htbp]
\centering
\input{figures/efficiency_sampling}
\caption{\textbf{Efficiency and accuracy at temperature 1.}
Efficiency metrics are ratios to \linear{}, accuracy differences are in percentage points.
Error bars show 95\% paired bootstrap confidence intervals. WritingBench measures speed only.}
\label{fig:efficiency-sampling}
\end{figure}
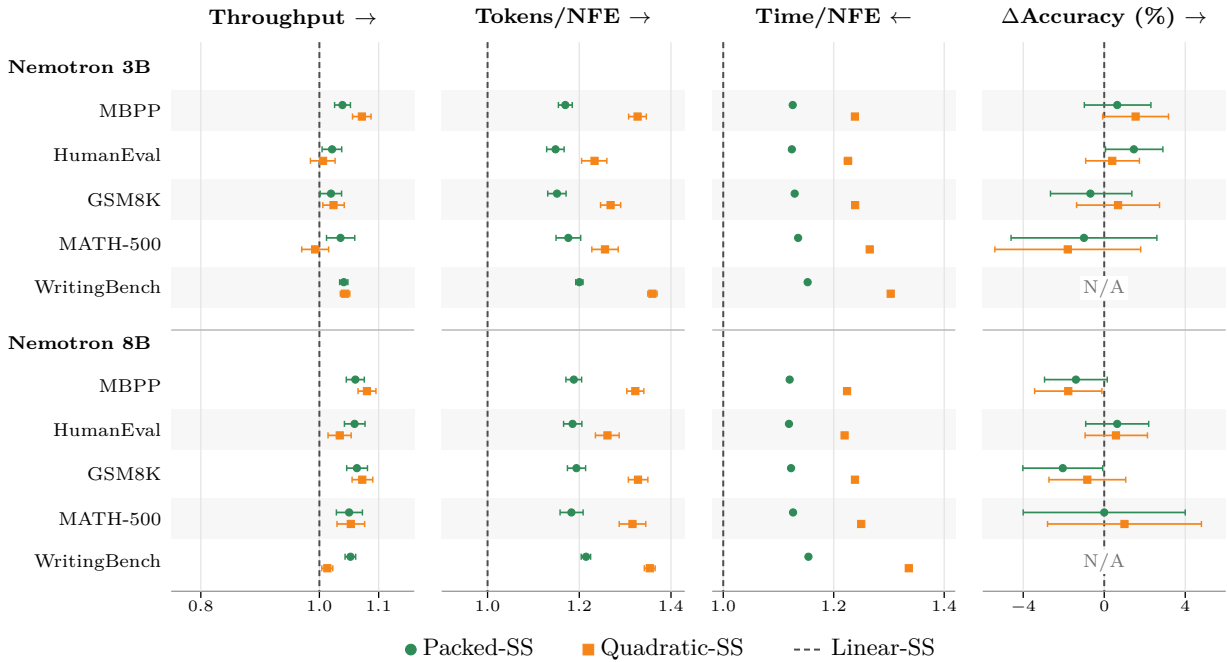

\paragraph{Results.}
Results are presented in \cref{fig:efficiency-sampling}, with numerical results in \cref{tab:rejection}. Compared with \linear{}, \method{} achieves higher throughput in all
ten settings, improving throughput by \(2.0\)--\(6.3\%\), with all ten 95\% confidence intervals for the throughput ratio above one.
The improvement again comes from increased token yield:
\method{} generates \(14.8\)--\(21.5\%\) more tokens per NFE,
which outweighs the \(11.9\)--\(15.4\%\) increase in decoder time per NFE.

Compared with \quadratic{}, \method{} shows comparable performance in terms of throughput.
Its throughput point estimate is higher in four settings and lower in
six. The paired 95\% confidence intervals lie above one in three
settings, below one in two, and include one in the remaining five.

\FloatBarrier

%% file: figures/efficiency_greedy.tex
\begingroup
\input{figures/efficiency_style}
\ifdefined\efficiencyFigureWidth\else\newdimen\efficiencyFigureWidth\fi
\efficiencyFigureWidth=\linewidth
\begin{tikzpicture}[x=\efficiencyFigureWidth,y=\efficiencyFigureWidth]
\path[use as bounding box] (0,0) rectangle (1,0.68500000);
\begin{axis}[efficiency axis,
  name=nemotron-0,
  at={(0.15000000\efficiencyFigureWidth,0.64500000\efficiencyFigureWidth)},
  width=0.19475000\efficiencyFigureWidth,height=0.44000000\efficiencyFigureWidth,
  xmin=0.75,xmax=1.16,ymin=-1.55,ymax=10.85,
  xtick={0.8,1,1.1},
  ytick={0,1,2,3,4,6.2,7.2,8.2,9.2,10.2},
  yticklabels={MBPP,HumanEval,GSM8K,MATH-500,WritingBench,MBPP,HumanEval,GSM8K,MATH-500,WritingBench},
]
\path[efficiency band] (axis cs:0.75,-0.46) rectangle (axis cs:1.16,0.46);
\path[efficiency band] (axis cs:0.75,1.54) rectangle (axis cs:1.16,2.46);
\path[efficiency band] (axis cs:0.75,3.54) rectangle (axis cs:1.16,4.46);
\path[efficiency band] (axis cs:0.75,6.74) rectangle (axis cs:1.16,7.66);
\path[efficiency band] (axis cs:0.75,8.74) rectangle (axis cs:1.16,9.66);
\draw[efficiency grid] (axis cs:0.8,-1.55)--(axis cs:0.8,10.85);
\draw[efficiency grid] (axis cs:1,-1.55)--(axis cs:1,10.85);
\draw[efficiency grid] (axis cs:1.1,-1.55)--(axis cs:1.1,10.85);
\draw[efficiency baseline] (axis cs:1,-1.55)--(axis cs:1,10.85);
\draw[efficiency separator] (axis cs:0.75,4.95)--(axis cs:1.16,4.95);
\node[anchor=east,xshift=-4pt,font=\efficiencyGroupFont] at (axis cs:0.75,-1) {Nemotron 3B};
\node[anchor=east,xshift=-4pt,font=\efficiencyGroupFont] at (axis cs:0.75,5.23) {Nemotron 8B};
\addplot[efficiency packed] coordinates {
  (1.044086472089,-0.130) += (0.031045501184,0) -= (0.036637445949,0) %
  (1.051802978774,0.870) += (0.019881153776,0) -= (0.018883315917,0) %
  (1.060521414794,1.870) += (0.010526780512,0) -= (0.011163556719,0) %
  (1.034558820732,2.870) += (0.013170182033,0) -= (0.013604272824,0) %
  (1.093119256707,3.870) += (0.006311149905,0) -= (0.005994817151,0) %
  (1.039777738969,6.070) += (0.015736722912,0) -= (0.015325302676,0) %
  (1.049800801327,7.070) += (0.028608205466,0) -= (0.027829853243,0) %
  (1.054118751328,8.070) += (0.011806183089,0) -= (0.011773041727,0) %
  (1.055590783185,9.070) += (0.015566487143,0) -= (0.015384178019,0) %
  (1.080729831930,10.070) += (0.007910359641,0) -= (0.007786783543,0) %
};
\addplot[efficiency quadratic] coordinates {
  (0.947675982199,0.130) += (0.027227556886,0) -= (0.032072694916,0) %
  (0.890626373580,1.130) += (0.030510000723,0) -= (0.028682527013,0) %
  (0.952022879691,2.130) += (0.010727528815,0) -= (0.010677927746,0) %
  (0.868264920725,3.130) += (0.013427174333,0) -= (0.013255774446,0) %
  (1.057084277020,4.130) += (0.006273747579,0) -= (0.006175537673,0) %
  (0.887517976045,6.330) += (0.022499345912,0) -= (0.022285647750,0) %
  (0.793509795789,7.330) += (0.029387230789,0) -= (0.030071904277,0) %
  (0.964954128012,8.330) += (0.010441942956,0) -= (0.010623797507,0) %
  (0.907455359329,9.330) += (0.014760590905,0) -= (0.014602359348,0) %
  (1.022199153815,10.330) += (0.008497617469,0) -= (0.008548285171,0) %
};
\end{axis}

\begin{axis}[efficiency axis,
  name=nemotron-1,
  at={(0.36675000\efficiencyFigureWidth,0.64500000\efficiencyFigureWidth)},
  width=0.19475000\efficiencyFigureWidth,height=0.44000000\efficiencyFigureWidth,
  xmin=0.9,xmax=1.43,ymin=-1.55,ymax=10.85,
  xtick={1,1.2,1.4},
  ytick={0,1,2,3,4,6.2,7.2,8.2,9.2,10.2},
  yticklabels={},
  yticklabel=\empty,
]
\path[efficiency band] (axis cs:0.9,-0.46) rectangle (axis cs:1.43,0.46);
\path[efficiency band] (axis cs:0.9,1.54) rectangle (axis cs:1.43,2.46);
\path[efficiency band] (axis cs:0.9,3.54) rectangle (axis cs:1.43,4.46);
\path[efficiency band] (axis cs:0.9,6.74) rectangle (axis cs:1.43,7.66);
\path[efficiency band] (axis cs:0.9,8.74) rectangle (axis cs:1.43,9.66);
\draw[efficiency grid] (axis cs:1,-1.55)--(axis cs:1,10.85);
\draw[efficiency grid] (axis cs:1.2,-1.55)--(axis cs:1.2,10.85);
\draw[efficiency grid] (axis cs:1.4,-1.55)--(axis cs:1.4,10.85);
\draw[efficiency baseline] (axis cs:1,-1.55)--(axis cs:1,10.85);
\draw[efficiency separator] (axis cs:0.9,4.95)--(axis cs:1.43,4.95);
\addplot[efficiency packed] coordinates {
  (1.147918777541,-0.130) += (0.034440886793,0) -= (0.040615364561,0) %
  (1.157653506700,0.870) += (0.021918651367,0) -= (0.020780596803,0) %
  (1.161634809476,1.870) += (0.011421553315,0) -= (0.011822343966,0) %
  (1.133566709731,2.870) += (0.014477585895,0) -= (0.014856817823,0) %
  (1.202996080124,3.870) += (0.006668219788,0) -= (0.006320514493,0) %
  (1.140107600681,6.070) += (0.017840673765,0) -= (0.017305925175,0) %
  (1.149342228143,7.070) += (0.032407501975,0) -= (0.032301785393,0) %
  (1.148531048446,8.070) += (0.012908700980,0) -= (0.012739236846,0) %
  (1.149902938435,9.070) += (0.017182555111,0) -= (0.016953126441,0) %
  (1.188346156879,10.070) += (0.008314371039,0) -= (0.008286249017,0) %
};
\addplot[efficiency quadratic] coordinates {
  (1.141574019279,0.130) += (0.033674066669,0) -= (0.039552197692,0) %
  (1.073288868428,1.130) += (0.036258253641,0) -= (0.034240345095,0) %
  (1.146789268148,2.130) += (0.012573398910,0) -= (0.012478327334,0) %
  (1.061094137072,3.130) += (0.016463601524,0) -= (0.016321142881,0) %
  (1.324494963921,4.130) += (0.007584196866,0) -= (0.007542519298,0) %
  (1.047372871744,6.330) += (0.026273910581,0) -= (0.025943553856,0) %
  (0.936775701602,7.330) += (0.035474349621,0) -= (0.036697574305,0) %
  (1.153575934679,8.330) += (0.012557878219,0) -= (0.012865496645,0) %
  (1.097429727017,9.330) += (0.017970369285,0) -= (0.018010407833,0) %
  (1.304708407213,10.330) += (0.009406184297,0) -= (0.009391875461,0) %
};
\end{axis}

\begin{axis}[efficiency axis,
  name=nemotron-2,
  at={(0.58350000\efficiencyFigureWidth,0.64500000\efficiencyFigureWidth)},
  width=0.19475000\efficiencyFigureWidth,height=0.44000000\efficiencyFigureWidth,
  xmin=0.98,xmax=1.42,ymin=-1.55,ymax=10.85,
  xtick={1,1.2,1.4},
  ytick={0,1,2,3,4,6.2,7.2,8.2,9.2,10.2},
  yticklabels={},
  yticklabel=\empty,
]
\path[efficiency band] (axis cs:0.98,-0.46) rectangle (axis cs:1.42,0.46);
\path[efficiency band] (axis cs:0.98,1.54) rectangle (axis cs:1.42,2.46);
\path[efficiency band] (axis cs:0.98,3.54) rectangle (axis cs:1.42,4.46);
\path[efficiency band] (axis cs:0.98,6.74) rectangle (axis cs:1.42,7.66);
\path[efficiency band] (axis cs:0.98,8.74) rectangle (axis cs:1.42,9.66);
\draw[efficiency grid] (axis cs:1,-1.55)--(axis cs:1,10.85);
\draw[efficiency grid] (axis cs:1.2,-1.55)--(axis cs:1.2,10.85);
\draw[efficiency grid] (axis cs:1.4,-1.55)--(axis cs:1.4,10.85);
\draw[efficiency baseline] (axis cs:1,-1.55)--(axis cs:1,10.85);
\draw[efficiency separator] (axis cs:0.98,4.95)--(axis cs:1.42,4.95);
\addplot[efficiency packed] coordinates {
  (1.099447994230,-0.130) += (0.003984979473,0) -= (0.003362208266,0) %
  (1.100637220147,0.870) += (0.003342022453,0) -= (0.003188154100,0) %
  (1.095343095643,1.870) += (0.001614138612,0) -= (0.001517620361,0) %
  (1.095700589484,2.870) += (0.001892542506,0) -= (0.001829083920,0) %
  (1.100516775953,3.870) += (0.001780828374,0) -= (0.001729689172,0) %
  (1.096491642350,6.070) += (0.003735702033,0) -= (0.003819679943,0) %
  (1.094819347337,7.070) += (0.005494371558,0) -= (0.005072102283,0) %
  (1.089565143394,8.070) += (0.001509122514,0) -= (0.001504078549,0) %
  (1.089345375833,9.070) += (0.001594966138,0) -= (0.001582332043,0) %
  (1.099577453837,10.070) += (0.002746981973,0) -= (0.002670946896,0) %
};
\addplot[efficiency quadratic] coordinates {
  (1.204603726086,0.130) += (0.003385004811,0) -= (0.003478833887,0) %
  (1.205094414747,1.130) += (0.003693811281,0) -= (0.003461119449,0) %
  (1.204581625727,2.130) += (0.001528045295,0) -= (0.001443107299,0) %
  (1.222085692677,3.130) += (0.001997607285,0) -= (0.001991546390,0) %
  (1.252970073167,4.130) += (0.001906783566,0) -= (0.001856868000,0) %
  (1.180114544171,6.330) += (0.003489092314,0) -= (0.003452242032,0) %
  (1.180547116838,7.330) += (0.005643237524,0) -= (0.005460150764,0) %
  (1.195472304011,8.330) += (0.001700665184,0) -= (0.001738236630,0) %
  (1.209348444234,9.330) += (0.001701205484,0) -= (0.001755540132,0) %
  (1.276373984799,10.330) += (0.005533455131,0) -= (0.005493413862,0) %
};
\end{axis}

\begin{axis}[efficiency axis,
  name=nemotron-3,
  at={(0.80025000\efficiencyFigureWidth,0.64500000\efficiencyFigureWidth)},
  width=0.19475000\efficiencyFigureWidth,height=0.44000000\efficiencyFigureWidth,
  xmin=-6,xmax=6,ymin=-1.55,ymax=10.85,
  xtick={-4,0,4},
  ytick={0,1,2,3,4,6.2,7.2,8.2,9.2,10.2},
  yticklabels={},
  yticklabel=\empty,
  xticklabel={\pgfmathprintnumber[fixed,precision=0]{\tick}},
]
\path[efficiency band] (axis cs:-6,-0.46) rectangle (axis cs:6,0.46);
\path[efficiency band] (axis cs:-6,1.54) rectangle (axis cs:6,2.46);
\path[efficiency band] (axis cs:-6,3.54) rectangle (axis cs:6,4.46);
\path[efficiency band] (axis cs:-6,6.74) rectangle (axis cs:6,7.66);
\path[efficiency band] (axis cs:-6,8.74) rectangle (axis cs:6,9.66);
\draw[efficiency grid] (axis cs:-4,-1.55)--(axis cs:-4,10.85);
\draw[efficiency grid] (axis cs:0,-1.55)--(axis cs:0,10.85);
\draw[efficiency grid] (axis cs:4,-1.55)--(axis cs:4,10.85);
\draw[efficiency baseline] (axis cs:0,-1.55)--(axis cs:0,10.85);
\draw[efficiency separator] (axis cs:-6,4.95)--(axis cs:6,4.95);
\addplot[efficiency packed] coordinates {
  (-0.264550264550,-0.130) += (1.058201058201,0) -= (1.058201058201,0) %
  (-1.219512195122,0.870) += (2.439024390244,0) -= (2.439024390244,0) %
  (-0.303260045489,1.870) += (0.833965125095,0) -= (0.833965125095,0) %
  (3.000000000000,2.870) += (2.200000000000,0) -= (2.200000000000,0) %
  (-0.264550264550,6.070) += (1.058201058201,0) -= (1.322751322751,0) %
  (1.219512195122,7.070) += (2.439024390244,0) -= (2.439024390244,0) %
  (0.151630022745,8.070) += (0.985595147839,0) -= (0.985595147839,0) %
  (1.400000000000,9.070) += (2.200000000000,0) -= (2.000000000000,0) %
};
\addplot[efficiency quadratic] coordinates {
  (-0.529100529101,0.130) += (1.058201058201,0) -= (1.058201058201,0) %
  (0.609756097561,1.130) += (1.219512195122,0) -= (0.609756097561,0) %
  (-0.151630022745,2.130) += (0.833965125095,0) -= (0.833965125095,0) %
  (2.800000000000,3.130) += (1.800000000000,0) -= (1.800000000000,0) %
  (-0.264550264550,6.330) += (0.793650793651,0) -= (0.793650793651,0) %
  (1.829268292683,7.330) += (2.439024390244,0) -= (1.829268292683,0) %
  (0.682335102350,8.330) += (1.137225170584,0) -= (1.061410159212,0) %
  (-0.200000000000,9.330) += (2.200000000000,0) -= (2.200000000000,0) %
};
\node[font=\efficiencyTaskFont,text=black!55,fill=white,inner sep=1pt] at (axis cs:0,4) {N/A};
\node[font=\efficiencyTaskFont,text=black!55,fill=white,inner sep=1pt] at (axis cs:0,10.2) {N/A};
\end{axis}

\begin{axis}[efficiency axis,
  name=setdlm-0,
  at={(0.15000000\efficiencyFigureWidth,0.17500000\efficiencyFigureWidth)},
  width=0.19475000\efficiencyFigureWidth,height=0.11000000\efficiencyFigureWidth,
  xmin=0.35,xmax=1.42,ymin=-1.25,ymax=1.8,
  xtick={0.5,1,1.4},
  ytick={0,1.12},
  yticklabels={{GSM8K, $W=8$},{GSM8K, $W=16$}},
]
\path[efficiency band] (axis cs:0.35,-0.45) rectangle (axis cs:1.42,0.45);
\draw[efficiency grid] (axis cs:0.5,-1.25)--(axis cs:0.5,1.8);
\draw[efficiency grid] (axis cs:1,-1.25)--(axis cs:1,1.8);
\draw[efficiency grid] (axis cs:1.4,-1.25)--(axis cs:1.4,1.8);
\draw[efficiency baseline] (axis cs:1,-1.25)--(axis cs:1,1.8);
\node[anchor=east,xshift=-4pt,font=\efficiencyGroupFont] at (axis cs:0.35,-0.96) {SetDLM};
\addplot[efficiency packed] coordinates {
  (1.059759426663,-0.230) += (0.001765736251,0) -= (0.001807058271,0) %
  (1.014713502667,0.890) += (0.002051651985,0) -= (0.002112126817,0) %
};
\addplot[efficiency quadratic] coordinates {
  (0.908521283465,0.000) += (0.002229030044,0) -= (0.002232223317,0) %
  (0.420831671009,1.120) += (0.002244045400,0) -= (0.002387520779,0) %
};
\addplot[efficiency native] coordinates {
  (0.813310094455,0.230) += (0.004687145814,0) -= (0.004815948361,0) %
  (0.815310508140,1.350) += (0.005650980580,0) -= (0.005452569250,0) %
};
\end{axis}

\begin{axis}[efficiency axis,
  name=setdlm-1,
  at={(0.36675000\efficiencyFigureWidth,0.17500000\efficiencyFigureWidth)},
  width=0.19475000\efficiencyFigureWidth,height=0.11000000\efficiencyFigureWidth,
  xmin=0.72,xmax=2.2,ymin=-1.25,ymax=1.8,
  xtick={1,1.5,2},
  ytick={0,1.12},
  yticklabels={},
  yticklabel=\empty,
]
\path[efficiency band] (axis cs:0.72,-0.45) rectangle (axis cs:2.2,0.45);
\draw[efficiency grid] (axis cs:1,-1.25)--(axis cs:1,1.8);
\draw[efficiency grid] (axis cs:1.5,-1.25)--(axis cs:1.5,1.8);
\draw[efficiency grid] (axis cs:2,-1.25)--(axis cs:2,1.8);
\draw[efficiency baseline] (axis cs:1,-1.25)--(axis cs:1,1.8);
\addplot[efficiency packed] coordinates {
  (1.088721830630,-0.230) += (0.001741382690,0) -= (0.001747399418,0) %
  (1.085347840221,0.890) += (0.001777269950,0) -= (0.001779333869,0) %
};
\addplot[efficiency quadratic] coordinates {
  (1.307890517867,0.000) += (0.003103211962,0) -= (0.003090149940,0) %
  (1.459263443934,1.120) += (0.006667051678,0) -= (0.006983601227,0) %
};
\addplot[efficiency native] coordinates {
  (0.811724802755,0.230) += (0.004721422199,0) -= (0.004849348686,0) %
  (0.795370881217,1.350) += (0.005459775061,0) -= (0.005253513619,0) %
};
\end{axis}

\begin{axis}[efficiency axis,
  name=setdlm-2,
  at={(0.58350000\efficiencyFigureWidth,0.17500000\efficiencyFigureWidth)},
  width=0.19475000\efficiencyFigureWidth,height=0.11000000\efficiencyFigureWidth,
  xmin=0.87,xmax=3.65,ymin=-1.25,ymax=1.8,
  xtick={1,2,3},
  ytick={0,1.12},
  yticklabels={},
  yticklabel=\empty,
]
\path[efficiency band] (axis cs:0.87,-0.45) rectangle (axis cs:3.65,0.45);
\draw[efficiency grid] (axis cs:1,-1.25)--(axis cs:1,1.8);
\draw[efficiency grid] (axis cs:2,-1.25)--(axis cs:2,1.8);
\draw[efficiency grid] (axis cs:3,-1.25)--(axis cs:3,1.8);
\draw[efficiency baseline] (axis cs:1,-1.25)--(axis cs:1,1.8);
\addplot[efficiency packed] coordinates {
  (1.027329225141,-0.230) += (0.000659841127,0) -= (0.000646925561,0) %
  (1.069610128740,0.890) += (0.001429500611,0) -= (0.001346145549,0) %
};
\addplot[efficiency quadratic] coordinates {
  (1.439581594478,0.000) += (0.001872494299,0) -= (0.001818416200,0) %
  (3.467570395628,1.120) += (0.004719448286,0) -= (0.004715767149,0) %
};
\addplot[efficiency native] coordinates {
  (0.998050815168,0.230) += (0.000925790326,0) -= (0.000921556325,0) %
  (0.975543517809,1.350) += (0.000857620081,0) -= (0.000844410696,0) %
};
\end{axis}

\begin{axis}[efficiency axis,
  name=setdlm-3,
  at={(0.80025000\efficiencyFigureWidth,0.17500000\efficiencyFigureWidth)},
  width=0.19475000\efficiencyFigureWidth,height=0.11000000\efficiencyFigureWidth,
  xmin=-14,xmax=14,ymin=-1.25,ymax=1.8,
  xtick={-10,0,10},
  ytick={0,1.12},
  yticklabels={},
  yticklabel=\empty,
  xticklabel={\pgfmathprintnumber[fixed,precision=0]{\tick}},
]
\path[efficiency band] (axis cs:-14,-0.45) rectangle (axis cs:14,0.45);
\draw[efficiency grid] (axis cs:-10,-1.25)--(axis cs:-10,1.8);
\draw[efficiency grid] (axis cs:0,-1.25)--(axis cs:0,1.8);
\draw[efficiency grid] (axis cs:10,-1.25)--(axis cs:10,1.8);
\draw[efficiency baseline] (axis cs:0,-1.25)--(axis cs:0,1.8);
\addplot[efficiency packed] coordinates {
  (0.000000000000,-0.230) += (0.000000000000,0) -= (0.000000000000,0) %
  (0.000000000000,0.890) += (0.000000000000,0) -= (0.000000000000,0) %
};
\addplot[efficiency quadratic] coordinates {
  (0.000000000000,0.000) += (0.000000000000,0) -= (0.000000000000,0) %
  (0.000000000000,1.120) += (0.000000000000,0) -= (0.000000000000,0) %
};
\addplot[efficiency native] coordinates {
  (-8.718726307809,0.230) += (2.350265352540,0) -= (2.350265352540,0) %
  (-7.050796057619,1.350) += (2.274450341168,0) -= (2.274450341168,0) %
};
\end{axis}

\node[font=\efficiencyPanelFont,anchor=south,yshift=5pt,inner sep=0pt] at (0.24737500,0.64500000) {Throughput $\rightarrow$};
\node[font=\efficiencyPanelFont,anchor=south,yshift=5pt,inner sep=0pt] at (0.46412500,0.64500000) {Tokens/NFE $\rightarrow$};
\node[font=\efficiencyPanelFont,anchor=south,yshift=5pt,inner sep=0pt] at (0.68087500,0.64500000) {Time/NFE $\leftarrow$};
\node[font=\efficiencyPanelFont,anchor=south,yshift=5pt,inner sep=0pt] at (0.89762500,0.64500000) {$\Delta$Accuracy (\%) $\rightarrow$};
\node[font=\efficiencyLegendFont,anchor=south,inner sep=0pt] at (.55,.010) {\efficiencyPackedKey\;\method{}\qquad\efficiencyQuadraticKey\;\quadratic{}\qquad\efficiencyNativeKey\;Native SetDLM\qquad\efficiencyLinearKey\;\linear{}};
\end{tikzpicture}
\endgroup

%% file: figures/efficiency_sampling.tex
\begingroup
\input{figures/efficiency_style}
\ifdefined\efficiencyFigureWidth\else\newdimen\efficiencyFigureWidth\fi
\efficiencyFigureWidth=\linewidth
\begin{tikzpicture}[x=\efficiencyFigureWidth,y=\efficiencyFigureWidth]
\path[use as bounding box] (0,0) rectangle (1,0.54500000);
\begin{axis}[efficiency axis,
  name=nemotron-0,
  at={(0.15000000\efficiencyFigureWidth,0.50500000\efficiencyFigureWidth)},
  width=0.19475000\efficiencyFigureWidth,height=0.44000000\efficiencyFigureWidth,
  xmin=0.75,xmax=1.16,ymin=-1.55,ymax=10.85,
  xtick={0.8,1,1.1},
  ytick={0,1,2,3,4,6.2,7.2,8.2,9.2,10.2},
  yticklabels={MBPP,HumanEval,GSM8K,MATH-500,WritingBench,MBPP,HumanEval,GSM8K,MATH-500,WritingBench},
]
\path[efficiency band] (axis cs:0.75,-0.46) rectangle (axis cs:1.16,0.46);
\path[efficiency band] (axis cs:0.75,1.54) rectangle (axis cs:1.16,2.46);
\path[efficiency band] (axis cs:0.75,3.54) rectangle (axis cs:1.16,4.46);
\path[efficiency band] (axis cs:0.75,6.74) rectangle (axis cs:1.16,7.66);
\path[efficiency band] (axis cs:0.75,8.74) rectangle (axis cs:1.16,9.66);
\draw[efficiency grid] (axis cs:0.8,-1.55)--(axis cs:0.8,10.85);
\draw[efficiency grid] (axis cs:1,-1.55)--(axis cs:1,10.85);
\draw[efficiency grid] (axis cs:1.1,-1.55)--(axis cs:1.1,10.85);
\draw[efficiency baseline] (axis cs:1,-1.55)--(axis cs:1,10.85);
\draw[efficiency separator] (axis cs:0.75,4.95)--(axis cs:1.16,4.95);
\node[anchor=east,xshift=-4pt,font=\efficiencyGroupFont] at (axis cs:0.75,-1) {Nemotron 3B};
\node[anchor=east,xshift=-4pt,font=\efficiencyGroupFont] at (axis cs:0.75,5.23) {Nemotron 8B};
\addplot[efficiency packed] coordinates {
  (1.038906811856,-0.130) += (0.013473654374,0) -= (0.013257342770,0) %
  (1.021386727580,0.870) += (0.016269896070,0) -= (0.016864054403,0) %
  (1.019592025515,1.870) += (0.017854869762,0) -= (0.018292613584,0) %
  (1.035690141222,2.870) += (0.023997884245,0) -= (0.023675671503,0) %
  (1.041162329523,3.870) += (0.006939798874,0) -= (0.006972088378,0) %
  (1.060616315747,6.070) += (0.015141008844,0) -= (0.015260209226,0) %
  (1.059312187084,7.070) += (0.017736445799,0) -= (0.017085076246,0) %
  (1.063283675370,8.070) += (0.017836355026,0) -= (0.017260766203,0) %
  (1.050115711808,9.070) += (0.022471464555,0) -= (0.021450350641,0) %
  (1.052430761945,10.070) += (0.008787076617,0) -= (0.009131619032,0) %
};
\addplot[efficiency quadratic] coordinates {
  (1.071721889030,0.130) += (0.015420724631,0) -= (0.015870358552,0) %
  (1.006219981688,1.130) += (0.020275764292,0) -= (0.021314979897,0) %
  (1.023821583946,2.130) += (0.018092118512,0) -= (0.017987161020,0) %
  (0.992834257929,3.130) += (0.022749986546,0) -= (0.022626662070,0) %
  (1.043242204333,4.130) += (0.008104387153,0) -= (0.007881362039,0) %
  (1.080199768649,6.330) += (0.015220391119,0) -= (0.015015909493,0) %
  (1.034333660681,7.330) += (0.019262419446,0) -= (0.019757807546,0) %
  (1.072337844843,8.330) += (0.017738979145,0) -= (0.017164946426,0) %
  (1.053016762732,9.330) += (0.023299875076,0) -= (0.023312695136,0) %
  (1.013390830294,10.330) += (0.009198451192,0) -= (0.009339631731,0) %
};
\end{axis}

\begin{axis}[efficiency axis,
  name=nemotron-1,
  at={(0.36675000\efficiencyFigureWidth,0.50500000\efficiencyFigureWidth)},
  width=0.19475000\efficiencyFigureWidth,height=0.44000000\efficiencyFigureWidth,
  xmin=0.9,xmax=1.43,ymin=-1.55,ymax=10.85,
  xtick={1,1.2,1.4},
  ytick={0,1,2,3,4,6.2,7.2,8.2,9.2,10.2},
  yticklabels={},
  yticklabel=\empty,
]
\path[efficiency band] (axis cs:0.9,-0.46) rectangle (axis cs:1.43,0.46);
\path[efficiency band] (axis cs:0.9,1.54) rectangle (axis cs:1.43,2.46);
\path[efficiency band] (axis cs:0.9,3.54) rectangle (axis cs:1.43,4.46);
\path[efficiency band] (axis cs:0.9,6.74) rectangle (axis cs:1.43,7.66);
\path[efficiency band] (axis cs:0.9,8.74) rectangle (axis cs:1.43,9.66);
\draw[efficiency grid] (axis cs:1,-1.55)--(axis cs:1,10.85);
\draw[efficiency grid] (axis cs:1.2,-1.55)--(axis cs:1.2,10.85);
\draw[efficiency grid] (axis cs:1.4,-1.55)--(axis cs:1.4,10.85);
\draw[efficiency baseline] (axis cs:1,-1.55)--(axis cs:1,10.85);
\draw[efficiency separator] (axis cs:0.9,4.95)--(axis cs:1.43,4.95);
\addplot[efficiency packed] coordinates {
  (1.169554136544,-0.130) += (0.015108035875,0) -= (0.015070692045,0) %
  (1.148225741144,0.870) += (0.018759599069,0) -= (0.019347919389,0) %
  (1.151425985375,1.870) += (0.019865318342,0) -= (0.020250303093,0) %
  (1.176048003859,2.870) += (0.027080185371,0) -= (0.026747431230,0) %
  (1.200210385827,3.870) += (0.007861357933,0) -= (0.007936380582,0) %
  (1.188111666208,6.070) += (0.017038701928,0) -= (0.017217123670,0) %
  (1.185366528548,7.070) += (0.020326898617,0) -= (0.019549519569,0) %
  (1.193792294780,8.070) += (0.020072486985,0) -= (0.019591614242,0) %
  (1.182818477151,9.070) += (0.025573637320,0) -= (0.024640620626,0) %
  (1.214738563042,10.070) += (0.010016849519,0) -= (0.010481109309,0) %
};
\addplot[efficiency quadratic] coordinates {
  (1.327199518213,0.130) += (0.019004086698,0) -= (0.019607138831,0) %
  (1.233303489920,1.130) += (0.026862368521,0) -= (0.028207571031,0) %
  (1.268265811330,2.130) += (0.022013058056,0) -= (0.021862947602,0) %
  (1.256108598091,3.130) += (0.028781530811,0) -= (0.028741041871,0) %
  (1.359494374090,4.130) += (0.010262840019,0) -= (0.009996065247,0) %
  (1.322276452255,6.330) += (0.018579702308,0) -= (0.018682710813,0) %
  (1.261543467594,7.330) += (0.025483255255,0) -= (0.026276641866,0) %
  (1.328055763338,8.330) += (0.021548752931,0) -= (0.021009531502,0) %
  (1.316035710655,9.330) += (0.028977874384,0) -= (0.028855301044,0) %
  (1.353786762988,10.330) += (0.011643366465,0) -= (0.011894889684,0) %
};
\end{axis}

\begin{axis}[efficiency axis,
  name=nemotron-2,
  at={(0.58350000\efficiencyFigureWidth,0.50500000\efficiencyFigureWidth)},
  width=0.19475000\efficiencyFigureWidth,height=0.44000000\efficiencyFigureWidth,
  xmin=0.98,xmax=1.42,ymin=-1.55,ymax=10.85,
  xtick={1,1.2,1.4},
  ytick={0,1,2,3,4,6.2,7.2,8.2,9.2,10.2},
  yticklabels={},
  yticklabel=\empty,
]
\path[efficiency band] (axis cs:0.98,-0.46) rectangle (axis cs:1.42,0.46);
\path[efficiency band] (axis cs:0.98,1.54) rectangle (axis cs:1.42,2.46);
\path[efficiency band] (axis cs:0.98,3.54) rectangle (axis cs:1.42,4.46);
\path[efficiency band] (axis cs:0.98,6.74) rectangle (axis cs:1.42,7.66);
\path[efficiency band] (axis cs:0.98,8.74) rectangle (axis cs:1.42,9.66);
\draw[efficiency grid] (axis cs:1,-1.55)--(axis cs:1,10.85);
\draw[efficiency grid] (axis cs:1.2,-1.55)--(axis cs:1.2,10.85);
\draw[efficiency grid] (axis cs:1.4,-1.55)--(axis cs:1.4,10.85);
\draw[efficiency baseline] (axis cs:1,-1.55)--(axis cs:1,10.85);
\draw[efficiency separator] (axis cs:0.98,4.95)--(axis cs:1.42,4.95);
\addplot[efficiency packed] coordinates {
  (1.125754613596,-0.130) += (0.001084430790,0) -= (0.001177867549,0) %
  (1.124183142526,0.870) += (0.001614250618,0) -= (0.001780635583,0) %
  (1.129300697300,1.870) += (0.001704291784,0) -= (0.001657174980,0) %
  (1.135521095597,2.870) += (0.001979724329,0) -= (0.001974302597,0) %
  (1.152760094938,3.870) += (0.001700759122,0) -= (0.001670615594,0) %
  (1.120208739549,6.070) += (0.000880615619,0) -= (0.000923302415,0) %
  (1.118996404460,7.070) += (0.001348532254,0) -= (0.001354167624,0) %
  (1.122741110800,8.070) += (0.001046902147,0) -= (0.001062402413,0) %
  (1.126369659887,9.070) += (0.001814293277,0) -= (0.001845620640,0) %
  (1.154221832890,10.070) += (0.002465921153,0) -= (0.002422143545,0) %
};
\addplot[efficiency quadratic] coordinates {
  (1.238380527447,0.130) += (0.002480098331,0) -= (0.002768895872,0) %
  (1.225679784108,1.130) += (0.006532152513,0) -= (0.006200634186,0) %
  (1.238756665436,2.130) += (0.002141014260,0) -= (0.002073594069,0) %
  (1.265174512321,3.130) += (0.002359257556,0) -= (0.002396764619,0) %
  (1.303143573414,4.130) += (0.001938586487,0) -= (0.001941785340,0) %
  (1.224103624748,6.330) += (0.001850162160,0) -= (0.002000625251,0) %
  (1.219667806966,7.330) += (0.004275111692,0) -= (0.004021909035,0) %
  (1.238467680428,8.330) += (0.001656573699,0) -= (0.001711698488,0) %
  (1.249776601126,9.330) += (0.002156088758,0) -= (0.002242701788,0) %
  (1.335897979850,10.330) += (0.005304987817,0) -= (0.005177222917,0) %
};
\end{axis}

\begin{axis}[efficiency axis,
  name=nemotron-3,
  at={(0.80025000\efficiencyFigureWidth,0.50500000\efficiencyFigureWidth)},
  width=0.19475000\efficiencyFigureWidth,height=0.44000000\efficiencyFigureWidth,
  xmin=-6,xmax=6,ymin=-1.55,ymax=10.85,
  xtick={-4,0,4},
  ytick={0,1,2,3,4,6.2,7.2,8.2,9.2,10.2},
  yticklabels={},
  yticklabel=\empty,
  xticklabel={\pgfmathprintnumber[fixed,precision=0]{\tick}},
]
\path[efficiency band] (axis cs:-6,-0.46) rectangle (axis cs:6,0.46);
\path[efficiency band] (axis cs:-6,1.54) rectangle (axis cs:6,2.46);
\path[efficiency band] (axis cs:-6,3.54) rectangle (axis cs:6,4.46);
\path[efficiency band] (axis cs:-6,6.74) rectangle (axis cs:6,7.66);
\path[efficiency band] (axis cs:-6,8.74) rectangle (axis cs:6,9.66);
\draw[efficiency grid] (axis cs:-4,-1.55)--(axis cs:-4,10.85);
\draw[efficiency grid] (axis cs:0,-1.55)--(axis cs:0,10.85);
\draw[efficiency grid] (axis cs:4,-1.55)--(axis cs:4,10.85);
\draw[efficiency baseline] (axis cs:0,-1.55)--(axis cs:0,10.85);
\draw[efficiency separator] (axis cs:-6,4.95)--(axis cs:6,4.95);
\addplot[efficiency packed] coordinates {
  (0.642479213908,-0.130) += (1.662887377173,0) -= (1.625094482237,0) %
  (1.463414634146,0.870) += (1.432926829268,0) -= (1.402439024390,0) %
  (-0.682335102350,1.870) += (2.047005307051,0) -= (1.971190295679,0) %
  (-1.000000000000,2.870) += (3.600000000000,0) -= (3.600000000000,0) %
  (-1.398337112623,6.070) += (1.549508692366,0) -= (1.549508692366,0) %
  (0.640243902439,7.070) += (1.554878048780,0) -= (1.554878048780,0) %
  (-2.047005307051,8.070) += (1.971190295679,0) -= (1.971190295679,0) %
  (0.000000000000,9.070) += (4.000000000000,0) -= (4.000000000000,0) %
};
\addplot[efficiency quadratic] coordinates {
  (1.549508692366,0.130) += (1.625094482237,0) -= (1.625094482237,0) %
  (0.396341463415,1.130) += (1.341463414634,0) -= (1.311737804878,0) %
  (0.682335102350,2.130) += (2.047005307051,0) -= (2.047005307051,0) %
  (-1.800000000000,3.130) += (3.600000000000,0) -= (3.600000000000,0) %
  (-1.776266061980,6.330) += (1.662887377173,0) -= (1.662887377173,0) %
  (0.579268292683,7.330) += (1.554878048780,0) -= (1.524390243902,0) %
  (-0.833965125095,8.330) += (1.895375284306,0) -= (1.895375284306,0) %
  (1.000000000000,9.330) += (3.800000000000,0) -= (3.800000000000,0) %
};
\node[font=\efficiencyTaskFont,text=black!55,fill=white,inner sep=1pt] at (axis cs:0,4) {N/A};
\node[font=\efficiencyTaskFont,text=black!55,fill=white,inner sep=1pt] at (axis cs:0,10.2) {N/A};
\end{axis}

\node[font=\efficiencyPanelFont,anchor=south,yshift=5pt,inner sep=0pt] at (0.24737500,0.50500000) {Throughput $\rightarrow$};
\node[font=\efficiencyPanelFont,anchor=south,yshift=5pt,inner sep=0pt] at (0.46412500,0.50500000) {Tokens/NFE $\rightarrow$};
\node[font=\efficiencyPanelFont,anchor=south,yshift=5pt,inner sep=0pt] at (0.68087500,0.50500000) {Time/NFE $\leftarrow$};
\node[font=\efficiencyPanelFont,anchor=south,yshift=5pt,inner sep=0pt] at (0.89762500,0.50500000) {$\Delta$Accuracy (\%) $\rightarrow$};
\node[font=\efficiencyLegendFont,anchor=south,inner sep=0pt] at (.55,.010) {\efficiencyPackedKey\;\method{}\qquad\efficiencyQuadraticKey\;\quadratic{}\qquad\efficiencyLinearKey\;\linear{}};
\end{tikzpicture}
\endgroup

%% file: sections/conclusion.tex
\section{Discussion and Limitations}
\label{sec:conclusion}

We presented Packed-SS, which overlaps drafting and verification in one
forward pass per iteration while evaluating only $O(W)$ positions.
Under greedy decoding and a monotonicity assumption, we showed that
Packed-SS requires no more forward passes than Linear-SS to
generate a prescribed number of tokens from the same prompt at the same width.
Experiments across multiple tasks and model families demonstrate that
Packed-SS can improve throughput over existing self-speculative decoding
schedules.

Several limitations and directions for future work remain. Our single-call
implementation requires shared-weight models supporting parallel drafting
and autoregressive verification. Combining this schedule with drafters
based on intermediate-layer representations remains unexplored.
The current efficiency analysis relies on an idealized assumption;
extending it to weaker conditions and stochastic sampling would broaden
its applicability. Finally, preparing only a subset of the continuation
branches used by \quadratic{} could provide intermediate designs between
Packed-SS and \quadratic{}, potentially improving upon the current algorithm.

%% file: sections/appendix.tex
\section{Notation}
\label{app:notation}

\Cref{tab:notation} collects the symbols used in \cref{sec:method}.
In left-to-right (LTR) decoding, \(d_t^{(i)}\) is indexed after \(P_t\), while \(z_t^{(i)}\) and \(v_t^{(i)}\) are indexed after \(P_{t-1}\). In the any-order procedure, position indices are fixed output positions after the prompt.

\begin{table}[H]
\caption{\textbf{Symbols.} Sizes and indices, decoder state, and distributions.}
\label{tab:notation}
\centering\footnotesize
\begin{tabular}{@{}l p{0.82\columnwidth}@{}}
\toprule
\multicolumn{2}{@{}l}{\textit{Sizes and indices}}\\
\(t\) & decoding iteration, starting at \(1\)\\
\(i,j\) & token position: relative in LTR, fixed in any-order decoding\\
\(r,s\) & rank in the verification order (any-order setting)\\
\(W\) & speculative width: the fixed number of future positions proposed by the LTR drafter, \(W\geq1\), and the size of the active window in the any-order setting. The schedule and \cref{thm:pathwise} are stated in \(W\) alone\\
\(B\) & Nemotron's block size, \(B=W+1\): the last prefix token and \(W\) draft positions\\
\(u_t\) & number of unresolved positions in the active window, \(u_t=|U_t|\leq W\) (any-order setting)\\
\(\kappa_t\) & available draft-prefix length in verification order \(\pi_t\), with \(0\leq\kappa_t\leq u_t\) (any-order setting)\\
\(h_t\) & number of consecutive accepted draft tokens in verification order\\
\(A_t\) & number of tokens accepted before applying the stopping rule; \(A_t=h_t+1\) for the next-token verifier and \(A_t=\min(h_t+1,u_t)\) for the any-order same-position verifier (\cref{app:any-order})\\
\(a\) & emitted token count after applying the stopping rule, \(a\leq A_t\) (LTR)\\
\(L\) & maximum number of generated tokens, excluding the prompt\\
\(E\) & last output position to generate in any-order decoding; initially \(L\), reduced to the position of an accepted EOS\\
\(n_t=|P_t|\) & LTR prefix length, including the prompt\\
\midrule
\multicolumn{2}{@{}l}{\textit{Decoder state}}\\
\(P\), \(P_t\) & prompt; verified prefix (LTR) or prompt followed by accepted position--token pairs (any-order), with \(P_0=P\)\\
\(D_t\) & draft available for verification at iteration \(t+1\), with \(D_0=\varnothing\); indexed after \(P_t\) in LTR, with \(|D_t|\leq W\). In any-order decoding, a position--token map on \(J_t\)\\
\(Z_t\) & new proposal: \(W\) positions after the input prefix in LTR, or positions in \(U_t\) in any-order decoding\\
\(v_t^{(i)}\) & verifier token at position \(i\)\\
\(U_t\), \(J_t\) & active unresolved positions; positions retained for the next draft, \(J_t\subseteq U_t\) (any-order)\\
\(\pi_t\) & verification order used by iteration \(t\)\\
\(\sigma\) & generation order of an any-order model\\
\(C_t\) & tokens accepted in iteration \(t\) after applying the stopping rule; tokens in LTR or ordered position--token pairs in any-order decoding\\
\midrule
\multicolumn{2}{@{}l}{\textit{Distributions}}\\
\(\mathcal V\), \(\theta\) & vocabulary; shared weights\\
\(p_\theta\) & target token distribution given the prompt, accepted tokens, and earlier draft tokens\\
\(q_\theta^{(i)}\) & drafter conditional at position \(i\): relative to its conditioning prefix in LTR, fixed in any-order decoding\\
\(q_t^{(j)}\) & fresh drafter conditional \(q_\theta^{(j)}(\cdot\mid P_{t-1})\), before prefix advancement (LTR)\\
\(Q_t^{(i)}\) & distribution stored with \(d_t^{(i)}\); \(Q_t^{(i)}=q_t^{(a+i)}\) after retaining the LTR suffix\\
\(p_t^{(i)}\) & verifier conditional at position \(i\) after \(P_{t-1}\) (LTR)\\
\midrule
\multicolumn{2}{@{}l}{\textit{Analysis (\cref{thm:pathwise})}}\\
\(c_k^\star\) & correctness threshold for output position \(k\); \(c_k^\star=k\) denotes no correct draft prediction before verification (\cref{ass:threshold})\\
\(x_k\), \(\widehat x_k(c)\) & target greedy token at output position \(k\); drafter's prediction from the first \(c\) verified output tokens\\
\(\ell_t\), \(c_t\) & verified output lengths after \(t\) Linear iterations and \(t\) Packed forward passes, respectively (\cref{app:pathwise})\\
\(F(c)\) & verified output length after fresh drafting and verification from \(c\) verified output tokens, before stopping (\cref{lem:F})\\
\(\nfe_S(L)\) & total model calls used by schedule \(S\) to generate \(L\) tokens after the prompt\\
\bottomrule
\end{tabular}
\end{table}

\section{Proofs}
\label{app:proofs}

\subsection{Pathwise dominance}
\label{app:pathwise}

We prove \cref{thm:pathwise} using two lemmas. The first describes the progress of fresh drafting and verification; the second compares this progress with two forward passes of \method{}. The comparisons below concern verified output lengths, excluding the prompt, before applying EOS or length truncation; stopping is handled at the end of the theorem's proof.

\begin{lemma}[Monotonicity of verified output length]
\label{lem:F}
Count verified output tokens after the prompt by \(c\), and define
\begin{equation}
F(c)=\min\bigl\{c+W+1,\ \min\{k>c:c_k^\star>c\}\bigr\},
\label{eq:F}
\end{equation}
with the minimum of an empty set read as \(+\infty\). Then \(F\) is nondecreasing. Under \cref{ass:threshold}, \(F(c)\) is the verified output length after freshly drafting \(W\) tokens from a context of \(c\) verified output tokens and verifying that draft, for any such context evaluated by either schedule. Thus, \(F(c)-c\) counts the newly accepted tokens, including the correction or bonus token.
\end{lemma}

\begin{proof}
We first identify the number of tokens accepted by fresh drafting and verification.
Fix a context of \(c\) verified output tokens.
If the draft contains an incorrect proposal, let \(k\) be its first incorrect position.
By \cref{ass:threshold}, this is the first position in \(c+1,\ldots,c+W\) with \(c_k^\star>c\).
Verification accepts the earlier proposals and adds the corrected token at \(k\), giving a verified output length of \(k\).
If all \(W\) proposals are correct, verification accepts them and adds a bonus token at \(c+W+1\).
These two cases give the expression for \(F(c)\) in \cref{eq:F}.

We next prove monotonicity.
Take \(c'\geq c\).
If \(c'\geq F(c)\), then
\[
F(c')\geq c'+1>F(c).
\]
It remains to consider \(c'<F(c)\).
For every position \(k\) satisfying \(c'<k<F(c)\), the definition of \(F(c)\) gives
\[
c_k^\star\leq c\leq c'.
\]
Thus, the first position with \(c_k^\star>c'\) cannot occur before \(F(c)\).
The other term in the minimum defining \(F(c')\) also satisfies
\[
c'+W+1\geq c+W+1\geq F(c).
\]
Both terms are therefore at least \(F(c)\), proving \(F(c')\geq F(c)\).
\end{proof}

\begin{lemma}[Verified output length after two forward passes]
\label{lem:two-call}
Under \cref{ass:threshold}, if a \method{} iteration begins with \(c\) verified output tokens and the next iteration ends with \(c''\), then \(c''\geq F(c)\).
\end{lemma}

\begin{proof}
Let \(c'\) be the verified output length after the first forward pass.
During this pass, \method{} verifies its stored draft and produces new proposals for positions \(c+1,\ldots,c+W\), using the context of length \(c\).
Every pass accepts at least one token, so \(c'>c\).

If \(c'\geq F(c)\), the desired inequality already holds, since the next pass can only increase the verified output length.
Otherwise, consider the positions
\[
c'+1,\ldots,F(c)-1.
\]
The new proposals at these positions are correct by the definition of \(F(c)\).
They also belong to the newly drafted block, because \(F(c)-1\leq c+W\).
Since these positions lie beyond the updated prefix, their proposals are retained for the next pass.

The next pass therefore accepts all of these retained proposals.
Its correction or bonus token is added at a position at least \(F(c)\), so its final verified output length satisfies \(c''\geq F(c)\).
This argument also covers an empty range of retained positions, since the next pass still accepts at least one token.
\end{proof}

\begin{proof}[Proof of \cref{thm:pathwise}]
We first compare the two schedules without EOS or length truncation.
Let \(\ell_t\) be the number of verified output tokens after \(t\) iterations of \linear{}, and let \(c_t\) be the corresponding number after \(t\) forward passes of \method{}.
Both counts exclude the prompt, and \(c_0=\ell_0=0\).
A \linear{} iteration drafts from its current verified prefix and then verifies that draft, so
\[
\ell_{t+1}=F(\ell_t).
\]

We show by induction that \(c_{2t}\geq\ell_t\) for every \(t\geq0\).
The claim holds at \(t=0\).
Suppose it holds for some \(t\).
Then
\[
c_{2t+2}\geq F(c_{2t})\geq F(\ell_t)=\ell_{t+1}.
\]
The first inequality follows from \cref{lem:two-call}; the second follows from the induction hypothesis and the monotonicity in \cref{lem:F}.
This proves the induction step.
The argument includes initialization: the first \method{} pass starts with an empty draft and generates one target token, so \cref{lem:two-call} applies from \(c_0=0\).

Now fix \(L>0\), and let \(T\) be the first \linear{} iteration for which \(\ell_T\geq L\).
Since each \linear{} iteration uses two forward passes, its total cost is \(2T\).
The inequality \(c_{2T}\geq\ell_T\geq L\) shows that \method{} reaches the same output length within those \(2T\) passes.
Therefore,
\[
\nfe_{\text{\method{}}}(L)\leq 2T
=\nfe_{\text{\linear{}}}(L).
\]
For \(L=0\), neither schedule makes a call.

Finally, both schedules produce the same greedy continuation.
If that continuation contains an EOS token, its first occurrence is therefore at the same position for both schedules.
Apply the same comparison to that position or the length limit \(L\), whichever is reached first.
Truncating the final block changes the returned output but does not remove the forward pass that produced it.
The same bound thus holds with EOS and length stopping.
\end{proof}

\subsection{Exact sampling}
\label{app:sampling}

\Cref{alg:packed-rejection} gives the full LTR sampling procedure corresponding to \cref{sec:method-general}.
Compared with \cref{alg:packed}, the decoder stores the sampling distribution \(Q_{t-1}^{(i)}\) with each draft token and replaces greedy matching with rejection sampling.

\begin{algorithm}[H]
\caption{\method{} with LTR rejection sampling}
\label{alg:packed-rejection}
\small
\begin{algorithmic}[1]
\REQUIRE prompt \(P\), speculative width \(W\), output length limit \(L\), verifier \(p_\theta\), drafter \(q_\theta\)
\STATE \(P_0\leftarrow P\), \ \(D_0,Q_0\leftarrow\varnothing\), \ \(t\leftarrow1\)
\WHILE{\(|P_{t-1}|-|P|<L\) and (\(t=1\) or the last token of \(P_{t-1}\) is not \(\mathrm{EOS}\))}
\item[] \textbf{In one model call:}
\begin{ALC@g}
\STATE \(p_t^{(i)}\leftarrow p_\theta(\cdot\mid P_{t-1},d_{t-1}^{(<i)})\) for \(i=1,\ldots,|D_{t-1}|+1\) \COMMENT{Verifier distributions}
\STATE \(q_t^{(j)}\leftarrow q_\theta^{(j)}(\cdot\mid P_{t-1})\) for \(j=1,\ldots,W\) \COMMENT{Drafter proposals}
\end{ALC@g}
\item[] \textbf{end}
\STATE \(u_i\sim\mathrm{Unif}(0,1)\), \ \(\alpha_i\leftarrow\min\{1,\,p_t^{(i)}(d_{t-1}^{(i)})/Q_{t-1}^{(i)}(d_{t-1}^{(i)})\}\) for \(i=1,\ldots,|D_{t-1}|\) \COMMENT{Acceptance tests}
\STATE \(h_t\leftarrow\max\{0\leq h\leq|D_{t-1}|:u_i<\alpha_i\text{ for all }i\leq h\}\) \COMMENT{Count accepted tokens}
\STATE \(c\sim R\), \quad \(R\propto\begin{cases}[p_t^{(h_t+1)}-Q_{t-1}^{(h_t+1)}]_+, & h_t<|D_{t-1}|,\\ p_t^{(|D_{t-1}|+1)}, & h_t=|D_{t-1}|.\end{cases}\) \COMMENT{Correction or bonus}
\STATE \(C_t\leftarrow(d_{t-1}^{(1)},\ldots,d_{t-1}^{(h_t)},c)\), \ \(A_t\leftarrow h_t+1\) \COMMENT{Add verifier token}
\STATE \(a\leftarrow\min\bigl(\{A_t,L+|P|-|P_{t-1}|\}\cup\{1\leq i\leq A_t:C_t^{(i)}=\mathrm{EOS}\}\bigr)\) \COMMENT{Ending condition}
\STATE \(C_t\leftarrow(C_t^{(1)},\ldots,C_t^{(a)})\) \COMMENT{Prune by stopping index}
\STATE \(P_t\leftarrow(P_{t-1},C_t)\) \COMMENT{Append to prefix}
\STATE \(Q_t^{(j)}\leftarrow q_t^{(a+j)}\), \ \(d_t^{(j)}\sim Q_t^{(j)}\) for \(j=1,\ldots,\max(0,W-a)\) \COMMENT{Sample retained draft}
\STATE \(D_t\leftarrow(d_t^{(1)},\ldots,d_t^{(\max(0,W-a))})\) \COMMENT{Prepare next draft}
\STATE \(t\leftarrow t+1\)
\ENDWHILE
\STATE \textbf{Output:} \(P_{t-1}\)
\end{algorithmic}
\end{algorithm}

We now prove \cref{thm:sampling}.
First, we show that the acceptance and correction steps produce the correct target distribution for one token.
Then, we show that each draft token still follows its stored sampling distribution when it is checked.

\paragraph{Conditioning information.}
Fix the prompt. Let \(\mathcal G_t\) collect the information used for conditioning immediately before sampling the draft \(D_t\).
It includes the verified prefix \(P_{t-1}\) at the start of call \(t\), the draft \(D_{t-1}\) checked in that call, and the distributions \(Q_{t-1}\) used to sample that draft.
It also includes the tokens \(C_t\) accepted after applying the stopping rule and the new drafter distributions \(q_t^{(1)},\ldots,q_t^{(W)}\), computed from \(P_{t-1}\).
Define
\[
\mathcal G_t:=\left(P_{t-1},D_{t-1},Q_{t-1},C_t,(q_t^{(j)})_{j=1}^{W}\right).
\]
Conditional on \(\mathcal G_t\), the advance \(a=|C_t|\), the updated prefix \(P_t\), and the stored distributions \(Q_t^{(i)}\) for the draft tokens are fixed.
All draft sampling, acceptance tests, and correction draws use independent random numbers.

Write \(u_{t+1,j}\) for the uniform random number used to test draft token \(d_t^{(j)}\) in call \(t+1\).
For \(1\leq i\leq|D_t|\), add the earlier draft tokens and the random numbers used to test them:
\[
\mathcal G_{t,i}:=\left(\mathcal G_t,(d_t^{(j)},u_{t+1,j})_{1\leq j<i}\right).
\]
Conditional on \(\mathcal G_{t,i}\), the preceding acceptance decisions are fixed, so whether verification reaches position \(i\) is also fixed.

\begin{lemma}[One-token correction using the original sampling distribution]
\label{lem:rs}
Consider the \(i\)-th draft token reached by left-to-right verification in call \(t+1\).
Let \(H=(P_t,d_t^{(<i)})\) be the verified prefix, including the draft tokens already accepted in this call.
Let \(P(x)=p(x\mid H)\) be the target distribution for the next token.
Let \(D=d_t^{(i)}\) be the draft token and \(Q=Q_t^{(i)}\) the distribution used to sample it.
Assume \(\Pr(D=x\mid\mathcal G_{t,i})=Q(x)\) for every token \(x\).
Accept \(D\) with probability \(\min\{1,P(D)/Q(D)\}\), and otherwise emit an independent draw from the normalized positive part of \(P-Q\).
Then, given \(\mathcal G_{t,i}\), the emitted token has distribution \(P\).
\end{lemma}

\begin{proof}
Condition on \(\mathcal G_{t,i}\), so that \(P\) and \(Q\) are fixed.
Let \(X\) denote the emitted token and \(A\) the event that the draft token is accepted.
We compute the probability of emitting a given token \(x\) through each of the two possible branches.

The acceptance branch contributes
\[
\Pr(X=x,A\mid\mathcal G_{t,i})=\min\{P(x),Q(x)\}.
\]
For \(Q(x)>0\), this follows by multiplying the proposal probability \(Q(x)\) by the acceptance probability \(\min\{1,P(x)/Q(x)\}\).
For \(Q(x)=0\), both sides are zero.
It follows that the total acceptance probability is
\[
\beta:=\Pr(A\mid\mathcal G_{t,i})=\sum_x\min\{P(x),Q(x)\}.
\]

The rejection probability is
\[
1-\beta=\sum_x\bigl(P(x)-\min\{P(x),Q(x)\}\bigr)
=\sum_x[P(x)-Q(x)]_+.
\]
When \(\beta<1\), this is exactly the normalizing constant of the correction distribution.
The rejection branch therefore contributes
\[
\Pr(X=x,A^{\mathrm c}\mid\mathcal G_{t,i})
=(1-\beta)\frac{[P(x)-Q(x)]_+}{1-\beta}
=[P(x)-Q(x)]_+.
\]
Adding the two contributions yields
\[
\Pr(X=x\mid\mathcal G_{t,i})
=\min\{P(x),Q(x)\}+[P(x)-Q(x)]_+
=P(x).
\]

If \(\beta=1\), then \(P=Q\) and rejection has probability zero, so the same conclusion holds without a correction draw.
Also, \(Q(D)>0\) with probability one because \(D\) is sampled from \(Q\), so the acceptance test never divides by zero.
Neither distribution needs to assign positive probability to every token.
\end{proof}

This is the standard speculative-sampling rule \citep{leviathan2023speculative,chen2023speculative}.
The next lemma shows that each draft token follows its stored sampling distribution when checked, as required by \cref{lem:rs}.

\begin{lemma}[Distribution of a draft token when checked]
\label{lem:suffix}
For each draft position \(1\leq i\leq|D_t|\) reached by verification in call \(t+1\), and every token \(x\),
\begin{equation}
\Pr(d_t^{(i)}=x\mid\mathcal G_{t,i})=Q_t^{(i)}(x).
\label{eq:draft-conditional}
\end{equation}
\end{lemma}

\begin{proof}
Condition on \(\mathcal G_t\).
The advance \(a\), the positions to retain, and the stored distributions \(Q_t^{(j)}=q_t^{(a+j)}\) are fixed.
Algorithm~\ref{alg:packed-rejection} samples the tokens in \(D_t\) independently from these distributions.
Thus, \(d_t^{(i)}\) has distribution \(Q_t^{(i)}\) and is independent of the earlier draft tokens given \(\mathcal G_t\).

By definition, \(\mathcal G_{t,i}\) adds only the earlier draft tokens and the random numbers used to test them to \(\mathcal G_t\).
These random numbers are independent of \((\mathcal G_t,D_t)\).
Conditioning on this additional information therefore leaves the distribution of \(d_t^{(i)}\) unchanged, proving \cref{eq:draft-conditional}.
Whether verification reaches position \(i\) is determined by \(\mathcal G_{t,i}\), so the same identity holds at every position that is checked.
\end{proof}

\begin{proof}[Proof of \cref{thm:sampling}]
Fix the prompt \(P\) and consider \cref{alg:packed-rejection} without EOS or length stopping.
Let \(X_k\) be its \(k\)-th output token after the prompt.
We prove by induction on \(m\) that
\[
\Pr(X_{1:m}=x_{1:m}\mid P)
=\prod_{k=1}^{m}p_\theta(x_k\mid P,x_{<k}).
\]
Since \(D_0=\varnothing\), the first call draws \(X_1\) directly from \(p_\theta(\cdot\mid P)\), proving the claim for \(m=1\).

Suppose the identity holds for \(m=k-1\), where \(k\geq2\).
Let \(t+1\) be the call that produces \(X_k\). Its position within that call is \(i=k-(|P_t|-|P|)\).
Thus, \(1\leq i\leq|D_t|+1\), with \(i=|D_t|+1\) denoting the bonus position.
Let \(\mathcal H_k\) collect \(t\), \(i\), \(\mathcal G_t\), the earlier draft tokens, and the random numbers used to test them:
\[
\mathcal H_k:=\left(t,i,\mathcal G_t,(d_t^{(j)},u_{t+1,j})_{1\leq j<i}\right).
\]
The first \(i-1\) draft tokens have been accepted, so
\[
(P,X_{<k})=(P_t,d_t^{(<i)}),\qquad
p_{t+1}^{(i)}=p_\theta(\cdot\mid P,X_{<k}).
\]

First suppose \(i\leq|D_t|\), so the algorithm checks a draft token.
For fixed \(t\) and \(i\), \(\mathcal G_{t,i}\) determines whether the preceding draft tokens are accepted and whether \(|P_t|-|P|+i=k\).
Hence \cref{lem:suffix} gives
\(\Pr(d_t^{(i)}=x\mid\mathcal H_k)=Q_t^{(i)}(x)\).
Applying \cref{lem:rs} with target \(p_{t+1}^{(i)}\) and proposal \(Q_t^{(i)}\) yields
\[
\Pr(X_k=x\mid\mathcal H_k)=p_\theta(x\mid P,X_{<k}).
\]
If \(i=|D_t|+1\), the algorithm draws a bonus token directly from \(p_{t+1}^{(i)}\), so the same equality holds.

Since \((P,X_{<k})\) is determined by \(\mathcal H_k\), the law of total expectation gives
\[
\begin{aligned}
\Pr(X_k=x\mid P,X_{<k})
&=\mathbb E\!\left[\Pr(X_k=x\mid\mathcal H_k)\mid P,X_{<k}\right]\\
&=p_\theta(x\mid P,X_{<k}).
\end{aligned}
\]
Multiplying by \(\Pr(X_{<k}=x_{<k}\mid P)\) and applying the induction hypothesis proves the identity for \(m=k\).

Now restore stopping. Let
\[
\tau:=\min\left(\{1\leq k\leq L:X_k=\mathrm{EOS}\}\cup\{L\}\right).
\]
Using the same random numbers, the stopped algorithm produces the same tokens through position \(\tau\) and returns \((P,X_{1:\tau})\).
Taking \(m=L\) in the identity above shows that \(X_{1:L}\) has the AR distribution.
Applying the same stopping rule therefore gives the same distribution of returned continuations.
\end{proof}

\paragraph{Greedy decoding.}
For greedy decoding, let each model assign probability one to the argmax token chosen by its fixed tie-breaking rule, and zero to all other tokens.
Then \cref{thm:sampling} gives the greedy correctness of \cref{alg:packed}, including EOS and length stopping.

\FloatBarrier
\begingroup
\raggedbottom
\setlength{\intextsep}{8pt plus 2pt minus 1pt}
\section{Any-Order Extension}
\label{app:any-order}

\label{app:any-order-construction}

A later token can be easier to predict than an earlier one.
In LTR verification, the first mismatch ends acceptance for that call, so an uncertain early token can prevent later, easier tokens from being accepted.
Checking tokens that are more likely to be accepted first may therefore let a call accept more draft tokens.

An any-order AR verifier enables this change: it conditions on tokens earlier in the chosen verification order, even when they occur later in the output sequence.
For a permutation \(\sigma\), write \(x_{\sigma(<r)}=(x_{\sigma(1)},\ldots,x_{\sigma(r-1)})\) for the tokens earlier in that order. The model defines
\[
p_\theta^\sigma(x_{1:L}\mid P)
=\prod_{r=1}^{L}p_\theta(x_{\sigma(r)}\mid P,x_{\sigma(<r)};\sigma).
\]
The order specifies the conditionals, so different orders can produce different greedy outputs.

Self-speculation with separate draft and verification passes already appears in \emph{Any-Subset Speculative Decoding} (ASSD) \citep{guo2026assd}.
ASSD uses the same model to draft tokens in parallel and verify them in a second pass, matching the two-call schedule we call \linear{}.
Prism also describes self-speculative decoding and, separately, choosing the next position by confidence \citep[Chapter~5]{li2025prism}.
Here we extend the packed schedule to such verification orders.

\paragraph{Procedure.}
In LTR decoding, accepted tokens form a prefix, and all unresolved positions lie after it.
With any-order verification, a later position can be accepted while an earlier one remains unresolved.
We therefore keep each accepted or draft token associated with its output position, numbered from 1 after the prompt.
At the start of iteration \(t\), \(P_{t-1}\) contains the prompt and accepted position--token pairs in acceptance order, and \(D_{t-1}\) maps positions to available draft tokens.
Position superscripts now refer to these fixed output positions rather than offsets after a prefix.

Each iteration starts its active window at the leftmost unresolved position and includes up to \(W\) consecutive output positions.
Positions whose tokens have already been accepted are skipped, leaving the set \(U_t\) of \(u_t=|U_t|\) unresolved positions to process.
Some positions in \(U_t\) may have no draft token yet; in particular, \(D_0=\varnothing\) in the first iteration.

The order policy then chooses a permutation \(\pi_t\) of \(U_t\) using information available before the model call, and this order remains fixed throughout the call.
Here, \(\pi_t(r)\) is the output position checked at rank \(r\); changing the order does not move a draft token to a different output position.
For example, confidence ordering checks positions in decreasing order of scores computed in an earlier iteration (see \cref{app:ordering-details} for the scores used in our experiments).

Verification follows \(\pi_t\), using draft tokens at earlier ranks as context.
Write \(\mathrm{dom}(D_{t-1})\) for the positions with available draft tokens, and let \(\kappa_t\) be the number of consecutive entries of \(\pi_t\) in this set before the first missing draft token.
A \emph{same-position} verifier predicts a token at each queried output position.
Writing \(p_{\theta,j}\) for its distribution at position \(j\) under the chosen order, its prediction at rank \(r\) is
\begin{equation}
v_t^{(\pi_t(r))}=\arg\max_x p_{\theta,\pi_t(r)}
\bigl(x\mid P_{t-1},(\pi_t(s),d_{t-1}^{(\pi_t(s))})_{s<r}\bigr),
\label{eq:any-verifier}
\end{equation}
for \(r=1,\ldots,\min(\kappa_t+1,u_t)\).
Accept matching draft tokens in this order until the first mismatch or missing draft token, then accept the verifier's prediction at that position.
If all \(u_t\) positions have draft tokens and all match, accept those tokens without an extra token outside the window.
Thus, if \(h_t\) draft tokens match consecutively, the number accepted before applying the stopping rule is \(A_t=\min(h_t+1,u_t)\).

In the same call as verification, \method{} drafts every position in \(U_t\) from \(P_{t-1}\).
After acceptance, predictions at positions that remain unresolved become the next draft; predictions at newly accepted positions are discarded.
For \quadratic{}, the proposal branches follow \(\pi_t\) instead of textual order, and the branch selected after verification supplies the next draft at the remaining unresolved positions.

\Cref{alg:setdlm-order} gives the formal procedure for \method{} with any-order greedy decoding.
Let \(E\) be the last output position to generate. It starts at the length limit \(L\); accepting EOS at position \(j\) changes it to \(\min(E,j)\).
The algorithm fills all remaining unresolved positions up to \(E\) before returning.
For an example with confidence-ordered verification on \setdlm{}, see \cref{app:setdlm-conf}.

In \cref{alg:setdlm-order}, \(\mathrm{dom}(P_t)\) denotes the output positions with accepted tokens in \(P_t\), excluding the prompt, and \(P_t(j)\) is the token stored at position \(j\).
The restriction \(\left.\cdot\right|_{\leq E}\) removes output pairs beyond \(E\), preserving the prompt and the order of the remaining pairs.
The completed sequence is returned with the prompt included and its continuation restored to textual order.

\begin{algorithm}[H]
\caption{\method{} with any-order greedy decoding}
\label{alg:setdlm-order}
\small
\begin{algorithmic}[1]
\REQUIRE prompt \(P\), speculative width \(W\geq1\), output length limit \(L>0\), verifier \(p_\theta\), drafter \(q_\theta\), order policy
\STATE \(P_0\leftarrow P\), \ \(D_0\leftarrow\varnothing\), \ \(E\leftarrow L\), \ \(t\leftarrow1\)
\WHILE{\(\{1,\ldots,E\}\nsubseteq\mathrm{dom}(P_{t-1})\)}
\STATE \(\ell\leftarrow\min\bigl(\{1,\ldots,E\}\setminus\mathrm{dom}(P_{t-1})\bigr)\) \COMMENT{Leftmost unresolved position}
\STATE \(U_t\leftarrow\{\ell,\ldots,\min(E,\ell+W-1)\}\setminus\mathrm{dom}(P_{t-1})\), \ \(u_t\leftarrow|U_t|\) \COMMENT{Active positions}
\STATE Choose \(\pi_t\) over \(U_t\) using the order policy and information available before this call
\STATE \(\kappa_t\leftarrow\max\{0\leq k\leq u_t:\pi_t(r)\in\mathrm{dom}(D_{t-1})\text{ for all }r\leq k\}\) \COMMENT{Available draft prefix}
\item[] \textbf{In one model call:}
\begin{ALC@g}
\STATE Compute \(v_t^{(\pi_t(r))}\) by \cref{eq:any-verifier} for \(r=1,\ldots,\min(\kappa_t+1,u_t)\) \COMMENT{Verifier predictions}
\STATE \(z_t^{(j)}\leftarrow\arg\max_x q_\theta^{(j)}(x\mid P_{t-1})\) for \(j\in U_t\) \COMMENT{Drafter proposals}
\end{ALC@g}
\item[] \textbf{end}
\STATE \(h_t\leftarrow\max\{0\leq h\leq\kappa_t:d_{t-1}^{(\pi_t(r))}=v_t^{(\pi_t(r))}\text{ for all }r\leq h\}\) \COMMENT{Count matches}
\STATE \(A_t\leftarrow\min(h_t+1,u_t)\), \ \(C_t\leftarrow\bigl((\pi_t(r),v_t^{(\pi_t(r))})\bigr)_{r=1}^{A_t}\) \COMMENT{Collect accepted tokens}
\STATE \(E\leftarrow\min\bigl(\{E\}\cup\{j:(j,\mathrm{EOS})\in C_t\}\bigr)\) \COMMENT{Ending condition}
\STATE \(C_t\leftarrow\left.C_t\right|_{\leq E}\), \ \(P_t\leftarrow\left.P_{t-1}\right|_{\leq E}\) \COMMENT{Discard positions after \(E\)}
\STATE \(P_t\leftarrow(P_t,C_t)\) \COMMENT{Append accepted tokens}
\STATE \(J_t\leftarrow\{j\in U_t\setminus\mathrm{dom}(P_t):j\leq E\}\) \COMMENT{Retained draft positions}
\STATE \(D_t\leftarrow\{j\mapsto z_t^{(j)}:j\in J_t\}\) \COMMENT{Prepare next draft}
\STATE \(t\leftarrow t+1\)
\ENDWHILE
\STATE \textbf{Output:} \((P,P_{t-1}(1),\ldots,P_{t-1}(E))\) \COMMENT{Textual order}
\end{algorithmic}
\end{algorithm}

\begin{proposition}[Greedy correctness under the chosen verification orders]
\label{thm:correctness-order}
Fix a run of \cref{alg:setdlm-order}.
Consider a \emph{serial reference} that starts from the same prompt and, in iteration \(t\), generates \(A_t\) tokens greedily at positions \(\pi_t(1),\ldots,\pi_t(A_t)\), conditioning each prediction on its own \(P_{t-1}\) and the tokens generated earlier in that group.
The reference uses \cref{alg:setdlm-order}'s verifier, tie-breaking rule, and rules for updating the accepted tokens and handling EOS.
Then the two procedures return the same output sequence.
\end{proposition}

\begin{proof}[Proof of \cref{thm:correctness-order}]
We show by induction that the two procedures have the same \(P_t\) and \(E\) after every iteration.
The claim holds initially, since both start from the prompt \(P\) with \(E=L\).

Suppose the claim holds before iteration \(t\).
Let \(y_r\) be the serial reference's token at position \(\pi_t(r)\) in this iteration.
We prove, in verification order, that
\[
y_r=v_t^{(\pi_t(r))},\qquad r=1,\ldots,A_t.
\]
At rank 1, both predictions condition only on \(P_{t-1}\), so they agree.
Now take \(r>1\) and suppose agreement holds at all earlier ranks.
Since \(r\leq A_t=\min(h_t+1,u_t)\), every \(s<r\) satisfies \(s\leq h_t\).
These earlier draft tokens match their verifier predictions by the definition of \(h_t\).
Thus,
\[
d_{t-1}^{(\pi_t(s))}=v_t^{(\pi_t(s))}=y_s,
\qquad s<r.
\]
Substituting these equalities into \cref{eq:any-verifier} gives
\[
\begin{aligned}
y_r
&=\arg\max_x p_{\theta,\pi_t(r)}
\bigl(x\mid P_{t-1},(\pi_t(s),y_s)_{s<r}\bigr)\\
&=v_t^{(\pi_t(r))}.
\end{aligned}
\]
This proves agreement at every rank up to \(A_t\), including the final verifier token when \(h_t<u_t\).

The two procedures therefore generate the same \(A_t\) position--token pairs before applying the stopping rule.
They find EOS at the same positions, update \(E\) identically, and remove the same pairs beyond \(E\).
The resulting values of \(P_t\) agree, completing the induction step.

Each iteration accepts at least one token at a previously unresolved position at or before the updated \(E\), while \(E\) never increases.
Hence the number of unresolved positions up to \(E\) strictly decreases, and both procedures terminate.
They return the same tokens in the same output-position order.
\end{proof}

\begin{remark}[Sampling under a specified generation order]
For a fixed generation order \(\sigma\), the argument in \cref{app:sampling} gives samples from \(p_\theta^\sigma\).
Choose the order before sampling the draft tokens, for example using proposal entropy or maximum proposal probability.
Sorting by the probabilities of sampled draft tokens can change their distributions when checked.
\end{remark}

\FloatBarrier
\FloatBarrier
\section{Attention Mask Layouts}
\label{app:implementation}
\label{app:boundaries}

\method{} combines verification of the current draft and generation of the next proposal in one model call.
This is achieved by a custom attention mask.
Below we show the packed inputs, masks, and cache updates for Nemotron and \setdlm{}.

\subsection{Nemotron: sequence packing and padding}
\label{app:nemotron-layout}

Let \(n_{t-1}=|P_{t-1}|\). Each call after initialization processes the last token of the verified prefix, \(P_{t-1}^{(n_{t-1})}\), followed by \(B-1\) draft slots; earlier prefix tokens are read from the cache. The first-call is slightly different and discussed in \cref{app:initialization}.

\paragraph{Physical layout.}
Let \(\bar D_{t-1}\) be \(D_{t-1}\) followed by enough masks to occupy \(B-1\) slots.
The \method{} input is
\[
[P_{t-1}^{(n_{t-1})},\bar D_{t-1}\mid P_{t-1}^{(n_{t-1})},\masktok^{B-1}],
\]
concatenated along the sequence axis. Earlier tokens are stored in KV cache and are not part of the input.
Its shape is \([1,2B]\).

Let \(T_B\) be the \(B\times B\) lower-triangular Boolean matrix including its diagonal, and \(J_B\) the \(B\times B\) all-ones matrix.
The following matrix specifies which rows in the current input may attend to each other:
\[
\begin{pmatrix}T_B&0\\0&J_B\end{pmatrix},
\]
and every row additionally sees the cached prefix \(P_{t-1}^{(<n_{t-1})}\).
The causal region yields the shifted next-token verifier predictions \citep{fu2026nemotron}.
The proposal region contains another copy of the last prefix token, whose output is unused, and \(B-1\) mask queries, whose outputs are the entries of \(Z_t\).
Padded verifier rows are evaluated but cannot extend acceptance, since verification stops at the first absent draft token, represented physically by a mask.

\paragraph{Cache update.}
After each call, the cache contains the verified prefix except its last token, which is processed in the next call.
Keys and values from the proposal region and unaccepted verifier inputs are discarded.

\FloatBarrier
\subsection{SetDLM: position-aligned queries and ordered context}
\label{app:setdlm-layout}

\paragraph{Duplicate masks.}
\setdlm{} predicts a token at the position of its input mask.
Its training construction represents each position twice: once with the token and once with a mask, both using the same position ID \citep[Section~5.3 and Appendix~K]{arriola2026setdiffusion}.
The mask can attend to tokens earlier in the generation order but cannot see its own target token or later tokens, allowing all conditional predictions to be computed in one pass.
We use the same construction for verification, placing draft tokens in the token entries and comparing them with predictions from the corresponding masks.
\method{} adds a second mask at each position to generate the next draft from the prompt and accepted tokens.

\paragraph{Packed input and attention.}
Let \(U_t\) contain the unresolved positions in the active window, and write \(u_t=|U_t|\).
Let \(N_t\) contain the correction token, if any, that has not yet been processed as an input.
Let \(O_t\) contain the previous draft \(D_{t-1}\), with masks for missing entries; \(M_t\) contain masks for verifier predictions; and \(R_t\) contain masks for new draft predictions.
The input is
\[
[N_t\mid O_t\mid M_t\mid R_t],
\qquad
|N_t|=p\in\{0,1\},\quad |O_t|=|M_t|=|R_t|=u_t.
\]
All regions attend to the cached keys and values of the prompt and accepted tokens.
The regions \(O_t,M_t,R_t\) also attend to \(N_t\), if present, and the row in \(N_t\) attends to itself.

Write \(O_r,M_r,R_r\) for the rows of \(O_t,M_t,R_t\) at verification rank \(r\).
Let \(\mathcal K_r\) contain the earlier ranks \(s<r\) in the order \(\pi_t\) that have a stored draft token.
Using \(X\leftarrow Y\) to mean that row \(X\) may attend to the rows in \(Y\), the remaining attention connections are
\[
O_r\leftarrow\{O_s:s\in\mathcal K_r\}\cup\{O_r\},
\qquad
M_r\leftarrow\{O_s:s\in\mathcal K_r\}\cup\{M_r\},
\qquad
R_r\leftarrow\{R_r\}.
\]
The three rows at rank \(r\) share the textual position ID \(\pi_t(r)\).
The physical call evaluates all \(u_t\) verifier queries, while the logical algorithm needs predictions only through the first missing draft token or disagreement.

\paragraph{Cache update.}
\label{app:any-order-stopping}
Only keys and values from \(N_t\) and accepted draft tokens in \(O_t\) are added to the cache.
A mask's keys and values do not represent its predicted token, so a correction is processed in the next call through \(N_{t+1}\).

\subsection{Input-size comparison}
\label{app:rows}

\Cref{tab:structures} compares the input sizes of these layouts with those of \linear{} and \quadratic{}.

\begin{table}[htbp]
  \caption{\textbf{Execution cost per iteration with cached prompt context.}
  Evaluated uncached rows count token representations processed by the transformer, including duplicates and evaluated padding, summed across calls.
  Nemotron uses block size \(B=W+1\). In \setdlm{}, \(u_t\leq W\) positions remain unresolved in the active window, and \(p=|N_t|\in\{0,1\}\) counts the correction token processed in this call; maximum row counts take \(u_t=W\) and \(p=1\).
  Prompt rows are processed once and depend on prompt length.}
  \label{tab:structures}
  \centering\small
  \setlength{\tabcolsep}{4pt}
  \begin{tabular}{lccrrc}
    \toprule
    & & \multicolumn{4}{c}{Evaluated uncached rows per iteration}\\
    \cmidrule(lr){3-6}
    & & & \multicolumn{2}{c}{\setdlm{} maximum} & \\
    \cmidrule(lr){4-5}
    Method & NFE/iteration & Nemotron & \(W=8\) & \(W=16\) & \setdlm{} general\\
    \midrule
    \linear{}    & 2 & \(2B\)      & 25 & 49  & \(p+3u_t\)\\
    \quadratic{} & 1 & \(B(B+1)\)  & 81 & 289 & \(p+2u_t+u_t^2\)\\
    \method{} (Ours) & 1 & \(2B\)  & 25 & 49  & \(p+3u_t\)\\
    \bottomrule
  \end{tabular}
\end{table}

For \setdlm{}, \linear{} processes \(p+u_t\) rows for drafting and \(2u_t\) for verification, caching \(N_t\) between calls.
\method{} processes the same total in one call.
\quadratic{} uses \(p+2u_t\) rows for verification and \(u_t\) proposal regions of \(u_t\) rows each.

\FloatBarrier

\section{Experimental Details}
\label{app:experimental-details}

\subsection{Hardware, models, and tasks}
All runs use batch size one on a single A100 80GB GPU.
Nemotron is evaluated in bfloat16 (BF16), and \setdlm{} in float32 (FP32). Nemotron uses left-to-right AR verification in every main setting.
Both scales run the same four tasks---HumanEval, MBPP, MATH-500, and GSM8K \citep{chen2021humaneval,austin2021mbpp,hendrycks2021math,cobbe2021gsm8k}---under greedy decoding and under rejection sampling.
Each cell draws on a fixed prompt set: 164 HumanEval problems, 378 MBPP problems, 500 MATH-500 problems, and all 1,319 GSM8K test questions, and these counts are the number of examples behind every mean and interval reported for that cell.
Both scales additionally run all 555 English WritingBench prompts \citep{wu2025writingbench} at both temperatures.
The \setdlm{} experiments allocate 1,024 output positions and use active windows of size 8 and 16; because its two windows come from separately adapted checkpoints, comparing them is not a same-weights window-size ablation.
At both temperatures, \linear{} and \method{} use Nemotron block size \(B=32\), while \quadratic{} uses \(B=8\), with speculative width \(W=B-1\).

\paragraph{Training horizon versus inference window.}
\label{app:smax32-semantics}

The \setdlm{} training construction uses a desired prediction budget \(S\in\{4,8,16\}\) and permits schedules reaching \(2S\) future positions \citep{arriola2026setdiffusion}.
Checkpoints are named by their maximum horizon.
Hence smax32 is the \(S=16\), maximum-horizon-32 model, and its matched active evaluation window is \(W=16\) rather than 32.

\subsection{Initialization and stopping}
\label{app:initialization}

\paragraph{Initial calls.}
In our Nemotron implementation, the drafting part of a regular call's input contains the last accepted token followed by \(W\) masks.
During initialization, the drafting part contains \(W+1\) masks, giving one extra draft position.
For \setdlm{}, \method{} builds the prompt cache in one counted call before entering the empty-draft loop of \cref{alg:setdlm-order}.

At temperature 1, our implementation first processes and caches the prompt, then samples the first output token from its last prediction. A second call proposes draft tokens conditioned on the prompt and this first output token.
Both calls count toward NFE and supply a full draft before verification, whereas \cref{alg:packed-rejection} starts from an empty draft.

\paragraph{Output caps.}
Generations are capped at 512 tokens for code and writing and 1,024 tokens for mathematics.

\paragraph{SetDLM stop sequences.}
An end marker in the \setdlm{} evaluation can consist of several tokens.
We recognize it only after all of those tokens have been accepted.
We set \(E\), the last output position, to the marker's final position and fill any earlier unresolved positions before returning.
If several markers are recognized, we use the one that ends earliest in the output.

\subsection{Prompts and scoring}
Nemotron uses zero-shot prompts: function completion for HumanEval, visible original test cases for MBPP, and requests for boxed final answers on the mathematics tasks.
MBPP uses the 378-problem EvalPlus subset and its original base tests.
MATH-500 uses boxed-answer exact match, and \setdlm{} GSM8K scoring compares normalized final answers.
The code and mathematics runs use the same prompts and graders at both temperatures.

For temperature-1 WritingBench runs, each complete query uses the checkpoint's user-message chat template with thinking disabled; prompts are not truncated.
Scored-task tables report accuracy or pass@1 at temperature 0 and mean per-response correctness at temperature 1.
WritingBench measures speed only, without judge-based scoring.

\subsection{Sampling and repetitions}
\label{app:sampling-protocol}

At temperature 1, rejection sampling can produce different generation trajectories for the same prompt, varying output length and decoding cost.
To reduce this source of variance on the code tasks, each method draws 20 responses per HumanEval prompt and 7 per MBPP prompt.
MATH-500, GSM8K, and WritingBench use one response per prompt and method.
Sampling uses float32 probability normalization and the full vocabulary, without top-\(k\) or top-\(p\) filtering.
For \quadratic{}, the verification result selects the branch matching the accepted draft prefix before its next draft tokens are sampled.

\subsection{Timing and confidence intervals}
Decoder wall time is CUDA-synchronized and includes prefill and all decoding work, excluding model loading, prompt preparation, tokenization, output-text decoding, and scoring.
Each model forward pass counts as one NFE, including prompt processing.
Generated-token counts include stop and special tokens and any generated reasoning text.

We compute 95\% percentile confidence intervals from 20,000 bootstrap resamples within each checkpoint, task, decoding mode, and verification order.
Each resample draws prompt indices with replacement, uses the same indices for all methods, and recomputes the reported statistics.
We bootstrap over prompts, using all repeated responses for each sampled prompt.
The intervals describe variation across prompts rather than timing noise for a fixed generation, and are not adjusted for multiple comparisons.

\subsection{SetDLM decoding settings}

\paragraph{Confidence scores.}
\label{app:ordering-details}
For confidence ordering, \method{} and \quadratic{} score each unresolved position by the probability assigned to its draft token when that token was proposed; \quadratic{} uses the selected proposal branch.
\linear{} uses the probability of the verifier's most likely token at each position in the previous iteration.
Positions are ranked by decreasing score, with ties broken by textual position.
Positions without a draft token or a finite score come last in textual order, so the first iteration uses LTR order.
Each order is fixed before the next model call; for \linear{}, this is the drafting call.

\paragraph{Native baseline.}
\label{app:native-baseline}
Native \setdlm{} uses the released sliding-window decoder with argmax token selection and confidence threshold \(0.9\).
It uses the same GSM8K prompts, number of output positions, window sizes, and stopping convention as the speculative schedules.

\FloatBarrier
\FloatBarrier
\section{Additional Experiments and Tables}
\label{app:records}

This appendix collects additional experiments, absolute measurements, and numerical values supporting the main figures.

\subsection{Nemotron block size and execution cost}
\label{app:width}

To select the block size for each schedule, we run a small probe on Nemotron 8B using the first ten GSM8K test questions (\cref{tab:width-probe}).
We choose the size with the highest token throughput (tokens/s) among those tested.
All runs use a total output limit of 512 new tokens.

\begin{table}[H]
  \caption{\textbf{Nemotron 8B GSM8K block-size probe.}
  Each row uses the same ten questions. Tokens/s is calculated from the displayed means. Wall is mean total decoder time per response, including prompt processing; total NFE and output-token counts are also per-response means. Accuracy is reported for comparison. Rows/call is the maximum number of input rows per call after initialization.}
  \label{tab:width-probe}
  \centering
  \small
  \setlength{\tabcolsep}{5pt}
  \begin{tabular}{lrrrrrrr}
    \toprule
    Method & \(B\) & Tokens/s & Rows/call & Wall (s) & Total NFE & Tokens & Accuracy (\%) \\
    \midrule
    \linear{} & 8 & 94.9 & 8 & 2.824 & 98.2 & 267.9 & 100 \\
    \linear{} & 16 & 113.4 & 16 & 2.448 & 84.6 & 277.5 & 100 \\
    \linear{} & 32 & 121.9 & 32 & 2.237 & 77.4 & 272.8 & 100 \\
    \method{} & 8 & 94.8 & 16 & 2.798 & 89.0 & 265.2 & 100 \\
    \method{} & 16 & 123.5 & 32 & 2.315 & 73.3 & 285.9 & 100 \\
    \method{} & 32 & 128.5 & 64 & 2.333 & 75.4 & 299.7 & 90 \\
    \quadratic{} & 8 & 122.3 & 72 & 2.232 & 64.9 & 272.9 & 100 \\
    \quadratic{} & 16 & 84.6 & 272 & 3.246 & 57.6 & 274.7 & 100 \\
    \quadratic{} & 32 & 25.4 & 1056 & 12.398 & 64.1 & 314.6 & 80 \\
    \bottomrule
  \end{tabular}
\end{table}

\FloatBarrier

\subsection{Proposal-only LoRA accounting}
\label{app:lora-details}

We evaluate the proposal-only low-rank adapter trained by the Nemotron authors \citep{fu2026nemotron} on Nemotron 8B and all GSM8K test questions, changing diffusion drafting while preserving the AR verifier.
The base and LoRA measurements in this study use a total output cap of 512 new tokens, including thinking and answer tokens.
This cap differs from the 1,024-token GSM8K setting in \cref{sec:experiments}, so absolute timings are not directly comparable with \cref{tab:nemotron}.
\method{} reduces NFE by 12.0\% relative to \linear{} but takes 1.749 rather than 1.698 seconds per example, a 3.0\% slowdown (\cref{tab:lora}).

\begin{table}[H]
  \caption{\textbf{Nemotron 8B GSM8K with proposal-only LoRA (\(n=1{,}319\)).}
  Means; model execution time and non-model decoder time sum to inclusive decoder wall time. Emitted length and throughput accompany latency, and accuracy is reported for comparison.}
  \label{tab:lora}
  \centering\footnotesize\setlength{\tabcolsep}{2.7pt}
  \begin{tabular}{lrrrrrrr}
    \toprule
    Method & Tokens & Total NFE & Model (s) & Non-model (s) & Wall (s) & Tokens/s & Accuracy (\%)\\
    \midrule
    \linear{}    & 272.2 & 55.04 & 1.561 & 0.137 & \textbf{1.698} & \textbf{160.32} & 85.90\\
    \quadratic{} & 273.3 & 54.44 & 1.781 & 0.343 & 2.123 & 128.72 & 86.35\\
    \method{}    & 274.2 & \textbf{48.44} & 1.570 & 0.179 & 1.749 & 156.78 & 85.82\\
    \bottomrule
  \end{tabular}
\end{table}

In calls with the adapter enabled, the implementation computes an update for every input row but applies it only to proposal rows.
Packed therefore computes 64 updates but applies only 32, whereas Linear places its 32 proposal rows in a separate call and disables the adapter for its 32-row verifier.
Quadratic applies updates to 64 proposal rows but computes them for all 72 input rows.

\begin{table}[H]
  \caption{\textbf{Nemotron 8B GSM8K model execution time per iteration.}
  Linear's iteration contains both calls.
  “Selected” counts rows receiving adapter updates; “evaluated” counts rows for which those updates are computed.}
  \label{tab:lora-iteration}
  \centering
  \small
  \setlength{\tabcolsep}{3.4pt}
  \begin{tabular}{lrrrrrr}
    \toprule
    Method & Query rows & Selected & Evaluated & Base ms & LoRA ms & Increase \\
    \midrule
    \linear{}    & \(32+32\) & 32 & 32 & 54.41 & 56.77 & 4.3\% \\
    \quadratic{} & 72        & 64 & 72 & 28.90 & 32.69 & 13.1\% \\
    \method{}    & 64        & 32 & 64 & 28.58 & 32.37 & 13.3\% \\
    \bottomrule
  \end{tabular}
\end{table}

With LoRA, mean model execution time is nearly equal for \method{} and \linear{} (1.570 versus 1.561 seconds).
After the first call, \method{} averages 47.44 iterations versus \linear{}'s 27.02, but its model time per iteration is 32.37 versus 56.77 milliseconds, so these differences almost cancel.
Non-model decoder time adds 0.179 seconds for \method{} against 0.137 for \linear{}, and the \method{}-minus-\linear{} wall-time difference is \(0.0513\) seconds with paired 95\% interval \([0.0243,0.0788]\).
Token throughput is also lower: \method{} emits 274.2 tokens on average against \linear{}'s 272.2, and its throughput ratio is \(0.978\) with paired interval \([0.969,0.987]\).
Shared, densely fine-tuned weights for drafting and verification would eliminate these extra LoRA computations and could reduce the latency penalty.
The Nemotron authors have not released such a checkpoint.

\FloatBarrier

\subsection{Model execution and non-model decoder time}
\label{app:time-breakdown}
To separate model evaluation from decoder overhead, we normalize each task's component times by its \linear{} total and then average tasks equally within each model scale, so that long-output tasks do not dominate the decomposition.
The decomposition covers GSM8K and MATH-500 only.
\begin{figure}[H]
\centering
\input{figures/nemotron_time_breakdown}
\caption{\textbf{Normalized Nemotron timing decomposition on GSM8K and MATH-500.}
Greedy decoding with a 1,024-token cap.
Each task's component times are divided by that task's \linear{} total, then the two tasks are averaged with equal weight within each model scale.
The dashed line marks \linear{}'s total.
Non-model decoder time covers draft token construction, selection, and state updates.}
\label{fig:nemotron-breakdown}
\end{figure}
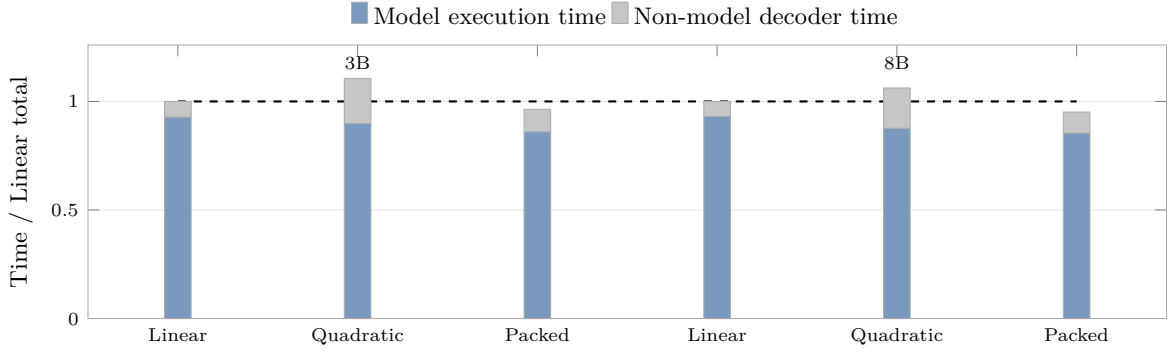
\method{} cuts model execution time to 0.860 and 0.854 of \linear{}'s total at 3B and 8B while raising non-model decoder time by about 40\%, for totals of 0.964 and 0.951.
\quadratic{} cuts model execution time comparably, to 0.898 and 0.876, but its non-model decoder time is 2.9 and 2.7 times \linear{}'s, so its totals exceed one at 1.106 and 1.062.

\FloatBarrier
\FloatBarrier

\subsection{Absolute measurements}
\label{app:efficiency-summaries}

\input{tables/nemotron_absolute_greedy}
\input{tables/nemotron_absolute_sampling}
\input{tables/setdlm_absolute_ltr}
\FloatBarrier

\subsection{Confidence-ordered verification on \setdlm{}}
\label{app:setdlm-conf}

\input{tables/setdlm_absolute}

\begin{table}[H]\caption{\textbf{Paired \setdlm{} efficiency comparisons.} Wall-time differences are \method{} minus baseline; throughput ratios are \method{}/baseline. Brackets give paired 95\% confidence intervals.}\label{tab:setdlm-paired-ci}\centering\footnotesize\setlength{\tabcolsep}{3pt}
\begin{tabular}{llrr}\toprule Setting & Baseline & \(\Delta\) wall (s) [95\% CI] & Throughput ratio [95\% CI]\\\midrule
\(W=8\) LTR & \linear{} & \(-0.1529\) [\(-0.1588\), \(-0.1470\)] & 1.060 [1.058, 1.062]\\
 & \quadratic{} & \(-0.4248\) [\(-0.4377\), \(-0.4124\)] & 1.166 [1.164, 1.169]\\
\(W=8\) Confidence & \linear{} & \(-0.2581\) [\(-0.2953\), \(-0.2217\)] & 1.102 [1.098, 1.107]\\
 & \quadratic{} & \(-0.3521\) [\(-0.3836\), \(-0.3198\)] & 1.140 [1.136, 1.143]\\
\(W=16\) LTR & \linear{} & \(-0.0382\) [\(-0.0432\), \(-0.0330\)] & 1.015 [1.013, 1.017]\\
 & \quadratic{} & \(-3.6686\) [\(-3.7869\), \(-3.5590\)] & 2.411 [2.399, 2.424]\\
\(W=16\) Confidence & \linear{} & \(-0.2428\) [\(-0.2757\), \(-0.2091\)] & 1.128 [1.123, 1.134]\\
 & \quadratic{} & \(-2.7033\) [\(-2.7766\), \(-2.6335\)] & 2.196 [2.185, 2.207]\\
\bottomrule\end{tabular}\end{table}

Under confidence order, \method{} remains fastest at both windows (\cref{tab:setdlm-order}), reducing \linear{}'s wall time by 9.5\% at \(W=8\) and 9.7\% at \(W=16\) and its total NFE by 12.4\% and 16.5\%.
Its throughput gains over \linear{} are 10.2\% and 12.8\%, and at \(W=16\) it reaches \(2.20\times\) \quadratic{}'s throughput; all eight \setdlm{} paired intervals favour \method{} (\cref{tab:setdlm-paired-ci}).
Two effects should be kept apart, because only one of them is a cost.
The three LTR schedules have identical per-example correctness labels within each checkpoint, so \method{}'s LTR speedup incurs no observed accuracy cost.
Changing the order can affect accuracy.
Confidence ordering lowers \method{}'s latency by a further 4.1\% at \(W=8\) and 13.1\% at \(W=16\) relative to LTR, but accuracy falls from 68.39\% to 65.35\% and from 66.57\% to 59.06\%; the second drop is large enough to give up the advantage over \setdlm{}'s own threshold decoder, which scores 59.51\% there.
Adaptive position choices alter the greedy trajectory even in exact arithmetic, and single-step confidence is in any case a weak criterion for choosing which tokens to accept \citep{chen2026ccd}.
We therefore report the ordering as a speed--accuracy trade-off of the decoding policy rather than a property of the schedule: \method{} gives the lowest latency in each regime.

\FloatBarrier

\subsection{Accuracy and output trajectories}
\label{app:accuracy-records}

We report paired accuracy differences to assess variation in task scores 
At temperature 1, the schedules draw different random samples, so their continuations and measured accuracies can differ.
Under greedy decoding, small floating-point differences can also change token choices and subsequent continuations.
\Cref{tab:score-comparisons} reports the Nemotron comparisons in percentage points, with paired 95\% confidence intervals.

For \setdlm{}, the LTR schedules have identical correctness labels despite different outputs on 5 of 1,319 smax16 questions and 13 of 1,319 smax32 questions.
Their accuracy gains over Native are \(8.72\) points at \(W=8\) (95\% CI \([6.37,11.07]\)) and \(7.05\) points at \(W=16\) (95\% CI \([4.78,9.33]\)).

\input{tables/nemotron_score_comparisons}
\FloatBarrier

\subsection{Paired efficiency comparisons}
\label{app:paired-efficiency}

\Cref{tab:paired-ci,tab:sampling-ci} report \method{} minus baseline wall time and \method{}/baseline throughput, with 95\% paired intervals. Negative differences and ratios above one favor \method{}.

\input{tables/nemotron_pairwise_greedy}
\input{tables/nemotron_pairwise_sampling}

All ten paired throughput intervals against \linear{} lie above one (\cref{tab:sampling-ci}), supporting a throughput advantage in the measured settings.
Against the \quadratic{} extension, three intervals lie above one, two lie below it, and five include it; the last group is inconclusive rather than evidence of equivalence.
\FloatBarrier

\paragraph{Reference schedules.}
At temperature 0, \quadratic{} has lower throughput than \linear{} on all eight code and mathematics settings but higher throughput on WritingBench; at temperature 1, it has higher throughput on seven of the eight code and mathematics settings.
Shorter accepted prefixes require more iterations, so eliminating a second call can outweigh the cost of its fixed 72-row pass (\cref{app:time-breakdown}).
On \setdlm{}, \linear{} and \method{} are faster than Native at both windows; \quadratic{} is faster only at \(W=8\) (\cref{tab:setdlm-length-details}).
\FloatBarrier
\FloatBarrier
\endgroup

%% file: figures/nemotron_time_breakdown.tex
\begin{tikzpicture}
\begin{axis}[width=0.96\textwidth,height=5.2cm,ybar stacked,bar width=10pt,
ymin=0,ymax=1.26,ylabel={Time / Linear total},
title={Temperature 0: GSM8K + MATH-500, cap 1,024},title style={font=\small\bfseries,yshift=22pt},
symbolic x coords={3B-L,3B-Q,3B-P,8B-L,8B-Q,8B-P},xtick=data,
xticklabels={Linear,Quadratic,Packed,Linear,Quadratic,Packed},
tick label style={font=\scriptsize},label style={font=\small},
legend style={at={(0.5,1.02)},anchor=south,legend columns=2,draw=none,font=\small},
axis line style={draw=gray!65},ymajorgrids=true,grid style={draw=gray!20},clip=false]
\addplot+[fill=linearblue!75,draw=gray!65] coordinates {
(3B-L,0.92727144)
(3B-Q,0.89817353)
(3B-P,0.85950845)
(8B-L,0.93088388)
(8B-Q,0.87583408)
(8B-P,0.85383016)
};
\addlegendentry{Model execution time}
\addplot+[fill=gray!45,draw=gray!65] coordinates {
(3B-L,0.07272856)
(3B-Q,0.20754559)
(3B-P,0.10435715)
(8B-L,0.06911612)
(8B-Q,0.18608035)
(8B-P,0.09712969)
};
\addlegendentry{Non-model decoder time}
\draw[black,dashed,thick] (axis cs:3B-L,1)--(axis cs:8B-P,1);
\node[font=\scriptsize] at (axis cs:3B-Q,1.18) {3B};
\node[font=\scriptsize] at (axis cs:8B-Q,1.18) {8B};
\end{axis}
\end{tikzpicture}

%% file: tables/nemotron_absolute_greedy.tex
\begin{table}[H]
\caption{\textbf{Nemotron absolute measurements at temperature 0.} Tokens, wall time, and NFE are means per response; rates use aggregate totals. Bold marks the highest tokens/s within each model and task.}
\label{tab:greedy-length-details}
\label{tab:nemotron}
\centering\footnotesize
\setlength{\tabcolsep}{4pt}
\begin{tabular}{@{}llrrrrrrr@{}}
\toprule
Task & Method & Tokens/s & \shortstack{Tokens/\\NFE} & \shortstack{Time/NFE\\(ms)} & \shortstack{Accuracy\\(\%)} & Tokens & \shortstack{Wall\\(s)} & NFE\\
\midrule
\multicolumn{9}{l}{\textit{Nemotron 3B}}\\
MBPP & \linear{} & 173.46 & 3.757 & 21.658 & 74.87 & 65.8 & 0.379 & 17.50\\
 & \quadratic{} & 164.38 & 4.288 & 26.089 & 74.34 & 65.9 & 0.401 & 15.36\\
 & \method{} & \textbf{181.10} & 4.312 & 23.812 & 74.60 & 67.0 & 0.370 & 15.53\\
\addlinespace[2pt]
HumanEval & \linear{} & 237.45 & 5.158 & 21.722 & 79.27 & 90.5 & 0.381 & 17.55\\
 & \quadratic{} & 211.48 & 5.536 & 26.177 & 79.88 & 91.1 & 0.431 & 16.46\\
 & \method{} & \textbf{249.75} & 5.971 & 23.908 & 78.05 & 90.0 & 0.361 & 15.08\\
\addlinespace[2pt]
GSM8K & \linear{} & 190.05 & 4.142 & 21.795 & 88.40 & 177.1 & 0.932 & 42.76\\
 & \quadratic{} & 180.93 & 4.750 & 26.253 & 88.25 & 178.9 & 0.989 & 37.66\\
 & \method{} & \textbf{201.55} & 4.812 & 23.873 & 88.10 & 179.1 & 0.888 & 37.22\\
\addlinespace[2pt]
MATH-500 & \linear{} & 222.40 & 4.928 & 22.160 & 55.40 & 489.0 & 2.199 & 99.23\\
 & \quadratic{} & 193.10 & 5.229 & 27.081 & 58.20 & 488.5 & 2.530 & 93.41\\
 & \method{} & \textbf{230.08} & 5.586 & 24.280 & 58.40 & 492.9 & 2.142 & 88.24\\
\addlinespace[2pt]
WritingBench & \linear{} & 84.56 & 1.864 & 22.038 & --- & 502.1 & 5.938 & 269.43\\
 & \quadratic{} & 89.39 & 2.468 & 27.613 & --- & 502.0 & 5.616 & 203.39\\
 & \method{} & \textbf{92.44} & 2.242 & 24.253 & --- & 502.9 & 5.441 & 224.34\\
\addlinespace[2pt]
\multicolumn{9}{l}{\textit{Nemotron 8B}}\\
MBPP & \linear{} & 151.96 & 4.421 & 29.095 & 85.45 & 56.9 & 0.375 & 12.87\\
 & \quadratic{} & 134.86 & 4.631 & 34.335 & 85.19 & 59.3 & 0.440 & 12.81\\
 & \method{} & \textbf{158.00} & 5.041 & 31.902 & 85.19 & 58.3 & 0.369 & 11.56\\
\addlinespace[2pt]
HumanEval & \linear{} & 195.71 & 5.695 & 29.099 & 80.49 & 94.2 & 0.481 & 16.54\\
 & \quadratic{} & 155.30 & 5.335 & 34.353 & 82.32 & 91.5 & 0.589 & 17.15\\
 & \method{} & \textbf{205.46} & 6.546 & 31.859 & 81.71 & 92.8 & 0.452 & 14.18\\
\addlinespace[2pt]
GSM8K & \linear{} & 127.39 & 3.744 & 29.389 & 90.52 & 313.2 & 2.459 & 83.66\\
 & \quadratic{} & 122.93 & 4.319 & 35.134 & 91.21 & 313.9 & 2.553 & 72.68\\
 & \method{} & \textbf{134.29} & 4.300 & 32.021 & 90.67 & 318.0 & 2.368 & 73.94\\
\addlinespace[2pt]
MATH-500 & \linear{} & 147.42 & 4.281 & 29.037 & 60.40 & 640.4 & 4.344 & 149.61\\
 & \quadratic{} & 133.78 & 4.698 & 35.116 & 60.20 & 630.7 & 4.715 & 134.26\\
 & \method{} & \textbf{155.61} & 4.922 & 31.632 & 61.80 & 634.7 & 4.079 & 128.94\\
\addlinespace[2pt]
WritingBench & \linear{} & 66.12 & 1.898 & 28.701 & --- & 501.7 & 7.587 & 264.35\\
 & \quadratic{} & 67.59 & 2.476 & 36.633 & --- & 501.5 & 7.420 & 202.55\\
 & \method{} & \textbf{71.46} & 2.255 & 31.559 & --- & 502.6 & 7.033 & 222.87\\
\addlinespace[2pt]
\bottomrule
\end{tabular}
\end{table}

%% file: tables/nemotron_absolute_sampling.tex
\begin{table}[H]
\caption{\textbf{Nemotron absolute measurements at temperature 1.} Tokens, wall time, and NFE are means per response; rates use aggregate totals. Bold marks the highest tokens/s within each model and task.}
\label{tab:rejection-length-details}
\label{tab:rejection}
\centering\footnotesize
\setlength{\tabcolsep}{4pt}
\begin{tabular}{@{}llrrrrrrr@{}}
\toprule
Task & Method & Tokens/s & \shortstack{Tokens/\\NFE} & \shortstack{Time/NFE\\(ms)} & \shortstack{Accuracy\\(\%)} & Tokens & \shortstack{Wall\\(s)} & NFE\\
\midrule
\multicolumn{9}{l}{\textit{Nemotron 3B}}\\
MBPP & \linear{} & 120.38 & 2.508 & 20.834 & 68.14 & 102.7 & 0.853 & 40.95\\
 & \quadratic{} & \textbf{129.01} & 3.328 & 25.800 & 69.69 & 108.7 & 0.843 & 32.67\\
 & \method{} & 125.06 & 2.933 & 23.454 & 68.78 & 101.8 & 0.814 & 34.72\\
\addlinespace[2pt]
HumanEval & \linear{} & 169.58 & 3.526 & 20.792 & 66.89 & 96.9 & 0.571 & 27.49\\
 & \quadratic{} & 170.63 & 4.348 & 25.484 & 67.29 & 95.6 & 0.561 & 22.00\\
 & \method{} & \textbf{173.20} & 4.048 & 23.374 & 68.35 & 97.7 & 0.564 & 24.12\\
\addlinespace[2pt]
GSM8K & \linear{} & 137.38 & 2.977 & 21.670 & 84.15 & 197.7 & 1.439 & 66.41\\
 & \quadratic{} & \textbf{140.65} & 3.776 & 26.843 & 84.84 & 199.2 & 1.416 & 52.76\\
 & \method{} & 140.07 & 3.428 & 24.471 & 83.47 & 198.4 & 1.416 & 57.88\\
\addlinespace[2pt]
MATH-500 & \linear{} & 150.13 & 3.266 & 21.753 & 53.80 & 518.2 & 3.451 & 158.66\\
 & \quadratic{} & 149.05 & 4.102 & 27.521 & 52.00 & 515.6 & 3.459 & 125.68\\
 & \method{} & \textbf{155.49} & 3.841 & 24.701 & 52.80 & 523.9 & 3.369 & 136.41\\
\addlinespace[2pt]
WritingBench & \linear{} & 74.96 & 1.575 & 21.012 & --- & 498.8 & 6.654 & 316.68\\
 & \quadratic{} & \textbf{78.20} & 2.141 & 27.382 & --- & 497.0 & 6.355 & 232.09\\
 & \method{} & 78.04 & 1.890 & 24.222 & --- & 500.1 & 6.409 & 264.58\\
\addlinespace[2pt]
\multicolumn{9}{l}{\textit{Nemotron 8B}}\\
MBPP & \linear{} & 84.45 & 2.316 & 27.426 & 76.76 & 129.9 & 1.538 & 56.08\\
 & \quadratic{} & \textbf{91.22} & 3.062 & 33.572 & 74.98 & 130.4 & 1.429 & 42.58\\
 & \method{} & 89.56 & 2.752 & 30.722 & 75.36 & 133.6 & 1.491 & 48.54\\
\addlinespace[2pt]
HumanEval & \linear{} & 107.78 & 2.944 & 27.316 & 70.70 & 120.3 & 1.116 & 40.87\\
 & \quadratic{} & 111.48 & 3.714 & 33.317 & 71.28 & 120.6 & 1.082 & 32.48\\
 & \method{} & \textbf{114.17} & 3.490 & 30.567 & 71.34 & 118.7 & 1.040 & 34.01\\
\addlinespace[2pt]
GSM8K & \linear{} & 85.66 & 2.341 & 27.324 & 87.34 & 351.0 & 4.098 & 149.97\\
 & \quadratic{} & \textbf{91.86} & 3.109 & 33.840 & 86.50 & 348.1 & 3.790 & 111.99\\
 & \method{} & 91.08 & 2.794 & 30.678 & 85.29 & 345.0 & 3.787 & 123.46\\
\addlinespace[2pt]
MATH-500 & \linear{} & 90.91 & 2.576 & 28.338 & 54.20 & 657.3 & 7.230 & 255.14\\
 & \quadratic{} & \textbf{95.73} & 3.391 & 35.417 & 55.20 & 658.5 & 6.879 & 194.22\\
 & \method{} & 95.47 & 3.047 & 31.919 & 54.20 & 664.6 & 6.962 & 218.10\\
\addlinespace[2pt]
WritingBench & \linear{} & 52.63 & 1.433 & 27.225 & --- & 500.9 & 9.518 & 349.59\\
 & \quadratic{} & 53.34 & 1.940 & 36.370 & --- & 496.5 & 9.308 & 255.94\\
 & \method{} & \textbf{55.39} & 1.741 & 31.424 & --- & 498.8 & 9.006 & 286.58\\
\addlinespace[2pt]
\bottomrule
\end{tabular}
\end{table}

%% file: tables/setdlm_absolute_ltr.tex
\begin{table}[H]
\caption{\textbf{SetDLM absolute measurements on GSM8K.} LTR verification and native decoding on all 1,319 test questions. Tokens, wall time, and NFE are means per response; rates use aggregate totals. Bold marks the highest tokens/s within each window.}
\label{tab:setdlm-length-details}
\centering\footnotesize
\setlength{\tabcolsep}{4pt}
\begin{tabular}{@{}llrrrrrrr@{}}
\toprule
Window & Method & Tokens/s & \shortstack{Tokens/\\NFE} & \shortstack{Time/NFE\\(ms)} & \shortstack{Accuracy\\(\%)} & Tokens & \shortstack{Wall\\(s)} & NFE\\
\midrule
$W=8$ & \linear{} & 94.93 & 2.610 & 27.496 & 68.39 & 257.2 & 2.709 & 98.53\\
 & \quadratic{} & 86.25 & 3.414 & 39.583 & 68.39 & 257.1 & 2.981 & 75.31\\
 & \method{} & \textbf{100.60} & 2.842 & 28.247 & 68.39 & 257.2 & 2.556 & 90.49\\
 & Native & 77.21 & 2.119 & 27.442 & 59.67 & 252.4 & 3.269 & 119.11\\
\addlinespace[3pt]
$W=16$ & \linear{} & 98.85 & 2.792 & 28.243 & 66.57 & 259.9 & 2.629 & 93.09\\
 & \quadratic{} & 41.60 & 4.074 & 97.934 & 66.57 & 260.4 & 6.260 & 63.92\\
 & \method{} & \textbf{100.30} & 3.030 & 30.209 & 66.57 & 259.9 & 2.591 & 85.77\\
 & Native & 80.59 & 2.220 & 27.552 & 59.51 & 253.6 & 3.147 & 114.21\\
\addlinespace[3pt]
\bottomrule
\end{tabular}
\end{table}

%% file: tables/setdlm_absolute.tex
\begin{table}[H]
\caption{\textbf{SetDLM measurements on GSM8K.} All 1,319 test questions. Top: LTR verification and native decoding; bottom: confidence-ordered verification. Tokens, inclusive wall time, and total NFE are means per response; rates use aggregate totals. Bold marks the highest tokens/s within each window and order.}
\label{tab:setdlm-order}
\centering\footnotesize
\setlength{\tabcolsep}{4pt}
\begin{tabular}{@{}llrrrrrrr@{}}
\toprule
Window & Method & Tokens/s & \shortstack{Tokens/\\NFE} & \shortstack{Time/NFE\\(ms)} & \shortstack{Accuracy\\(\%)} & Tokens & \shortstack{Wall\\(s)} & NFE\\
\midrule
\multicolumn{9}{l}{\textit{LTR verification and native decoding}}\\
$W=8$ & \linear{} & 94.93 & 2.610 & 27.496 & 68.39 & 257.2 & 2.709 & 98.53\\
 & \quadratic{} & 86.25 & 3.414 & 39.583 & 68.39 & 257.1 & 2.981 & 75.31\\
 & \method{} & \textbf{100.60} & 2.842 & 28.247 & 68.39 & 257.2 & 2.556 & 90.49\\
 & Native & 77.21 & 2.119 & 27.442 & 59.67 & 252.4 & 3.269 & 119.11\\
\addlinespace[3pt]
$W=16$ & \linear{} & 98.85 & 2.792 & 28.243 & 66.57 & 259.9 & 2.629 & 93.09\\
 & \quadratic{} & 41.60 & 4.074 & 97.934 & 66.57 & 260.4 & 6.260 & 63.92\\
 & \method{} & \textbf{100.30} & 3.030 & 30.209 & 66.57 & 259.9 & 2.591 & 85.77\\
 & Native & 80.59 & 2.220 & 27.552 & 59.51 & 253.6 & 3.147 & 114.21\\
\addlinespace[3pt]
\multicolumn{9}{l}{\textit{Confidence-ordered verification}}\\
$W=8$ & \linear{} & 93.93 & 2.586 & 27.535 & 65.05 & 254.5 & 2.709 & 98.39\\
 & \quadratic{} & 90.83 & 3.470 & 38.200 & 65.58 & 254.6 & 2.803 & 73.38\\
 & \method{} & \textbf{103.52} & 2.943 & 28.431 & 65.35 & 253.8 & 2.451 & 86.21\\
\addlinespace[3pt]
$W=16$ & \linear{} & 99.06 & 2.794 & 28.206 & 53.07 & 247.0 & 2.493 & 88.40\\
 & \quadratic{} & 50.90 & 4.462 & 87.664 & 60.80 & 252.1 & 4.954 & 56.51\\
 & \method{} & \textbf{111.76} & 3.408 & 30.491 & 59.06 & 251.5 & 2.251 & 73.81\\
\addlinespace[3pt]
\bottomrule
\end{tabular}
\par\smallskip
\begin{minipage}{\linewidth}\footnotesize
The training horizons are \(S_{\max}=16\) and \(32\) for \(W=8\) and \(16\), respectively. For \(W=8\) and \(16\), respectively, maximum recurrent rows/call are \((16,32)\) for \linear{}, \((81,289)\) for \quadratic{}, \((25,49)\) for \method{}, and \((16,28)\) for Native. Native commits at confidence threshold \(0.9\) and has no verification order (\cref{app:native-baseline}).
\end{minipage}
\end{table}

%% file: tables/nemotron_score_comparisons.tex
\begin{table}[H]
\caption{\textbf{Paired Nemotron accuracy differences.}}
\label{tab:score-comparisons}
\centering\footnotesize
\setlength{\tabcolsep}{3pt}
\begin{tabular}{@{}lccc@{}}
\toprule
Setting & \shortstack{\method{}\\$-$ \linear{}} & \shortstack{\method{}\\$-$ \quadratic{}} & \shortstack{\linear{}\\$-$ \quadratic{}}\\
\midrule
\multicolumn{4}{l}{\textit{Temperature 0}}\\
3B MBPP & -0.26 [-1.32, +0.79] & +0.26 [-0.79, +1.32] & +0.53 [-0.53, +1.59]\\
\addlinespace[2pt]
3B HumanEval & -1.22 [-3.66, +1.22] & -1.83 [-4.27, +0.00] & -0.61 [-1.83, +0.00]\\
\addlinespace[2pt]
3B GSM8K & -0.30 [-1.14, +0.53] & -0.15 [-0.99, +0.68] & +0.15 [-0.68, +0.99]\\
\addlinespace[2pt]
3B MATH-500 & +3.00 [+0.80, +5.20] & +0.20 [-1.60, +2.00] & -2.80 [-4.60, -1.00]\\
\addlinespace[2pt]
8B MBPP & -0.26 [-1.59, +0.79] & +0.00 [-0.79, +0.79] & +0.26 [-0.53, +1.06]\\
\addlinespace[2pt]
8B HumanEval & +1.22 [-1.22, +3.66] & -0.61 [-3.05, +1.22] & -1.83 [-4.27, +0.00]\\
\addlinespace[2pt]
8B GSM8K & +0.15 [-0.83, +1.14] & -0.53 [-1.59, +0.53] & -0.68 [-1.82, +0.38]\\
\addlinespace[2pt]
8B MATH-500 & +1.40 [-0.60, +3.60] & +1.60 [-0.40, +3.80] & +0.20 [-2.00, +2.40]\\
\addlinespace[2pt]
\multicolumn{4}{l}{\textit{Temperature 1}}\\
3B MBPP & +0.64 [-0.98, +2.31] & -0.91 [-2.61, +0.83] & -1.55 [-3.17, +0.08]\\
\addlinespace[2pt]
3B HumanEval & +1.46 [+0.06, +2.90] & +1.07 [-0.34, +2.47] & -0.40 [-1.74, +0.92]\\
\addlinespace[2pt]
3B GSM8K & -0.68 [-2.65, +1.36] & -1.36 [-3.34, +0.61] & -0.68 [-2.73, +1.36]\\
\addlinespace[2pt]
3B MATH-500 & -1.00 [-4.60, +2.60] & +0.80 [-2.80, +4.40] & +1.80 [-1.80, +5.40]\\
\addlinespace[2pt]
8B MBPP & -1.40 [-2.95, +0.15] & +0.38 [-1.32, +2.12] & +1.78 [+0.11, +3.44]\\
\addlinespace[2pt]
8B HumanEval & +0.64 [-0.91, +2.20] & +0.06 [-1.49, +1.59] & -0.58 [-2.13, +0.95]\\
\addlinespace[2pt]
8B GSM8K & -2.05 [-4.02, -0.08] & -1.21 [-3.26, +0.83] & +0.83 [-1.06, +2.73]\\
\addlinespace[2pt]
8B MATH-500 & +0.00 [-4.00, +4.00] & -1.00 [-4.60, +2.80] & -1.00 [-4.80, +2.80]\\
\addlinespace[2pt]
\bottomrule
\end{tabular}
\end{table}

%% file: tables/nemotron_pairwise_greedy.tex
\begin{table}[H]
\caption{\textbf{Paired Nemotron efficiency comparisons at temperature 0.}}
\label{tab:paired-ci}
\centering\footnotesize
\setlength{\tabcolsep}{3pt}
\begin{tabular}{@{}llrr@{}}
\toprule
Setting & Baseline & $\Delta$ wall (s) [95\% CI] & Throughput ratio [95\% CI]\\
\midrule
3B MBPP & \linear{} & -0.0092 [-0.0268, +0.0123] & 1.044 [1.007, 1.075]\\
 & \quadratic{} & -0.0308 [-0.0563, -0.0027] & 1.102 [1.079, 1.127]\\
3B HumanEval & \linear{} & -0.0207 [-0.0341, -0.0100] & 1.052 [1.033, 1.072]\\
 & \quadratic{} & -0.0703 [-0.0886, -0.0539] & 1.181 [1.149, 1.213]\\
3B GSM8K & \linear{} & -0.0434 [-0.0605, -0.0243] & 1.061 [1.049, 1.071]\\
 & \quadratic{} & -0.1003 [-0.1187, -0.0798] & 1.114 [1.103, 1.125]\\
3B MATH-500 & \linear{} & -0.0565 [-0.1145, +0.0040] & 1.035 [1.021, 1.048]\\
 & \quadratic{} & -0.3872 [-0.4531, -0.3221] & 1.192 [1.174, 1.209]\\
3B WritingBench & \linear{} & -0.4969 [-0.5335, -0.4604] & 1.093 [1.087, 1.099]\\
 & \quadratic{} & -0.1753 [-0.2068, -0.1435] & 1.034 [1.029, 1.040]\\
8B MBPP & \linear{} & -0.0058 [-0.0222, +0.0156] & 1.040 [1.024, 1.056]\\
 & \quadratic{} & -0.0711 [-0.0921, -0.0520] & 1.172 [1.144, 1.200]\\
8B HumanEval & \linear{} & -0.0294 [-0.0509, -0.0077] & 1.050 [1.022, 1.078]\\
 & \quadratic{} & -0.1372 [-0.1634, -0.1121] & 1.323 [1.274, 1.376]\\
8B GSM8K & \linear{} & -0.0908 [-0.1499, -0.0318] & 1.054 [1.042, 1.066]\\
 & \quadratic{} & -0.1856 [-0.2433, -0.1277] & 1.092 [1.082, 1.104]\\
8B MATH-500 & \linear{} & -0.2656 [-0.3560, -0.1778] & 1.056 [1.040, 1.071]\\
 & \quadratic{} & -0.6361 [-0.7332, -0.5367] & 1.163 [1.145, 1.181]\\
8B WritingBench & \linear{} & -0.5538 [-0.6101, -0.4958] & 1.081 [1.073, 1.089]\\
 & \quadratic{} & -0.3866 [-0.4433, -0.3300] & 1.057 [1.050, 1.065]\\
\bottomrule
\end{tabular}
\end{table}

%% file: tables/nemotron_pairwise_sampling.tex
\begin{table}[H]
\caption{\textbf{Paired Nemotron efficiency comparisons at temperature 1.}}
\label{tab:sampling-ci}
\centering\footnotesize
\setlength{\tabcolsep}{3pt}
\begin{tabular}{@{}llrr@{}}
\toprule
Setting & Baseline & $\Delta$ wall (s) [95\% CI] & Throughput ratio [95\% CI]\\
\midrule
3B MBPP & \linear{} & -0.0389 [-0.0678, -0.0107] & 1.039 [1.026, 1.052]\\
 & \quadratic{} & -0.0287 [-0.0832, +0.0248] & 0.969 [0.956, 0.984]\\
3B HumanEval & \linear{} & -0.0076 [-0.0250, +0.0093] & 1.021 [1.005, 1.038]\\
 & \quadratic{} & +0.0033 [-0.0121, +0.0189] & 1.015 [0.997, 1.034]\\
3B GSM8K & \linear{} & -0.0227 [-0.0832, +0.0387] & 1.020 [1.001, 1.037]\\
 & \quadratic{} & +0.0000 [-0.0640, +0.0652] & 0.996 [0.979, 1.013]\\
3B MATH-500 & \linear{} & -0.0820 [-0.2440, +0.0791] & 1.036 [1.012, 1.060]\\
 & \quadratic{} & -0.0895 [-0.2481, +0.0661] & 1.043 [1.021, 1.066]\\
3B WritingBench & \linear{} & -0.2455 [-0.3181, -0.1699] & 1.041 [1.034, 1.048]\\
 & \quadratic{} & +0.0536 [-0.0187, +0.1283] & 0.998 [0.990, 1.006]\\
8B MBPP & \linear{} & -0.0467 [-0.0970, +0.0024] & 1.061 [1.045, 1.076]\\
 & \quadratic{} & +0.0619 [+0.0072, +0.1164] & 0.982 [0.969, 0.995]\\
8B HumanEval & \linear{} & -0.0767 [-0.1104, -0.0443] & 1.059 [1.042, 1.077]\\
 & \quadratic{} & -0.0423 [-0.0786, -0.0059] & 1.024 [1.007, 1.042]\\
8B GSM8K & \linear{} & -0.3104 [-0.4833, -0.1363] & 1.063 [1.046, 1.081]\\
 & \quadratic{} & -0.0025 [-0.1731, +0.1618] & 0.992 [0.976, 1.007]\\
8B MATH-500 & \linear{} & -0.2686 [-0.5450, +0.0071] & 1.050 [1.029, 1.073]\\
 & \quadratic{} & +0.0831 [-0.1779, +0.3394] & 0.997 [0.978, 1.018]\\
8B WritingBench & \linear{} & -0.5121 [-0.6381, -0.3880] & 1.052 [1.043, 1.061]\\
 & \quadratic{} & -0.3029 [-0.4490, -0.1578] & 1.039 [1.030, 1.047]\\
\bottomrule
\end{tabular}
\end{table}